\documentclass[twoside,11pt]{article}

\usepackage{blindtext}

\usepackage{jmlr2e}

\usepackage{hyperref}
\hypersetup{
    colorlinks,
    linkcolor={red!50!black},
    citecolor={blue!50!black},
    urlcolor={blue!80!black}
}

\usepackage{url}
\usepackage[T1]{fontenc}    
\usepackage{hyperref}       
\usepackage{url}            
\usepackage{booktabs}       
\usepackage{amsfonts}       
\usepackage{nicefrac}       
\usepackage{microtype}      
\usepackage{xcolor}         
\usepackage{wrapfig}

\usepackage{microtype}
\usepackage{graphicx}
\usepackage{subfigure}
\usepackage{booktabs} 
\usepackage{amsfonts}
\usepackage{amssymb}
\usepackage{amsmath}
\usepackage{sansmath}
\usepackage{MnSymbol,wasysym}
\usepackage{subfigure}
\usepackage{tcolorbox}
\newtheorem{assumption}[theorem]{Assumption}
\newtheorem{fact}[theorem]{Fact}
\usepackage{mathtools}
\usepackage{enumitem}
\newcommand{\eat}[1]{}

\usepackage{hyperref}

\usepackage{bm}
\usepackage{multirow}
\newtcolorbox[list inside=tcb]{mytcb}[1][]{#1}

\DeclareMathOperator{\prox}{prox}
\DeclareMathOperator{\proj}{proj}
\DeclareMathOperator{\dist2}{dist}

\allowdisplaybreaks[4]
\usepackage{caption}

\usepackage{lastpage}

\ShortHeadings{Fast and Provably Convergent Algorithms for GW in Graph Data}{Fast and Provably Convergent Algorithm for GW in Graph Data}
\firstpageno{1}

\begin{document}

\title{Fast and Provably Convergent Algorithms for Gromov-Wasserstein  in  Graph Data }

\author{\name Jiajin Li \email jiajinli@stanford.edu \\
       \addr 
       Stanford University
       \AND
       \name Jianheng Tang \email jtangbf@connect.ust.hk  \\
       \addr Hong Kong University of Science and
Technology 
\AND 
\name Lemin Kong \email lkong@se.cuhk.edu.hk \\
\addr The Chinese University of Hong Kong
\AND
\name Huikang Liu \email liuhuikang@sufe.edu.cn\\
\addr Shanghai
University of Finance and Economics.
\AND 
\name Jia Li \email jialee@ust.hk \\
\addr Hong Kong University of Science and
Technology 
\AND
\name Anthony Man-Cho So \email manchoso@se.cuhk.edu.hk\\
\addr The Chinese University of Hong Kong
 \AND
\name Jose Blanchet \email jose.blanchet@stanford.edu\\
\addr Stanford University 
}

\editor{}

\maketitle

\begin{abstract}
In this paper, we study the design and analysis of a class of efficient algorithms for computing the Gromov-Wasserstein (GW) distance tailored to large-scale graph learning tasks.  Armed with the Luo-Tseng error bound condition~\citep{luo1992error}, two proposed algorithms, called Bregman Alternating Projected Gradient (BAPG) and hybrid Bregman Proximal Gradient (hBPG) 
 enjoy the convergence guarantees.
Upon task-specific properties, our analysis further provides novel theoretical insights to guide how to select the best-fit method.
As a result, we are able to provide comprehensive experiments to validate the effectiveness of our methods on a host of tasks, including graph alignment, graph partition, and shape matching. In terms of both wall-clock time and modeling performance, the proposed methods achieve state-of-the-art results.
\end{abstract}


\section{Introduction}


The Gromov-Wasserstein (GW) distance provides a flexible way to compare and couple probability distributions supported on different metric spaces. 
As such, we have witnessed a fast-growing body of literature that applies the GW distance to various structural data analysis tasks, e.g., 2D/3D shape matching~\citep{peyre2016gromov,memoli2004comparing,memoli2009spectral}, molecule analysis~\citep{vayer2018fused,titouan2019optimal}, graph alignment and partition~\citep{chowdhury2019gromov,xu2019gromov,xu2019scalable,chowdhury2021generalized,gao2021unsupervised}, graph embedding and classification~\citep{vincent2021online, xu2022representing}, generative modeling~\citep{bunne2019learning,xu2021learning}, to name a few.

Although the GW distance has received much attention in both machine learning and data science communities, there are still few results that focus on the design of practically efficient algorithms with provable convergence guarantees.
Existing algorithms either are double-loop ({i.e., requiring another iterative algorithm as  a subroutine at each iteration}) or do not have any theoretical justification. Recently, \citet{solomon2016entropic}
has proposed an entropy-regularized iterative sinkhorn projection algorithm called eBPG, which is proven to converge under the Kurdyka-\L{}ojasiewicz framework~\citep{attouch2010proximal,attouch2013convergence}. 
However, eBPG has several crucial drawbacks that prevent its widespread adoption. First, instead of tackling the GW problem directly, eBPG addresses an entropic regularized GW objective, whose regularization parameter affects the model performance dramatically.
Second, eBPG has to solve the entropic optimal transport problem~\citep{peyre2019computational,benamou2015iterative} at each iteration via another iterative algorithm~\citep{cuturi2013sinkhorn}. Thus, eBPG is neither computationally efficient nor practically robust. To solve the GW problem directly, 
\cite{xu2019gromov} proposes the Bregman projected gradient (BPG). Unfortunately, BPG is still a double-loop algorithm that also invokes the sinkhorn as its subroutine. Notably, due to the lack of the entropic regularizer, the inner solver suffered from the numerical instability issue.
On another front, \citet{titouan2019optimal,memoli2011gromov} have introduced the Frank-Wolfe (FW) method (see~\citep{jaggi2013revisiting,lacoste2016convergence} for recent treatments) to solve the GW problem. Nevertheless, in their implementation, they still rely on off-the-shelf linear programming solvers and line-search schemes, which are not well-suited for even medium size tasks.  
To get rid of the computational burden and sensitivity issue arising from the double-loop scheme,
it is only recently that \cite{xu2019gromov} has further developed a simple heuristic single-loop method (BPG-S) based on BPG, which presented an attractive empirical performance on the node correspondence task. However, such a heuristic still lacks theoretical support, and thus it is not necessarily guaranteed to perform well (or even converge) under naturally noisy observations. 
\begin{figure}[t!]
	\centering
	\includegraphics[width=\textwidth]{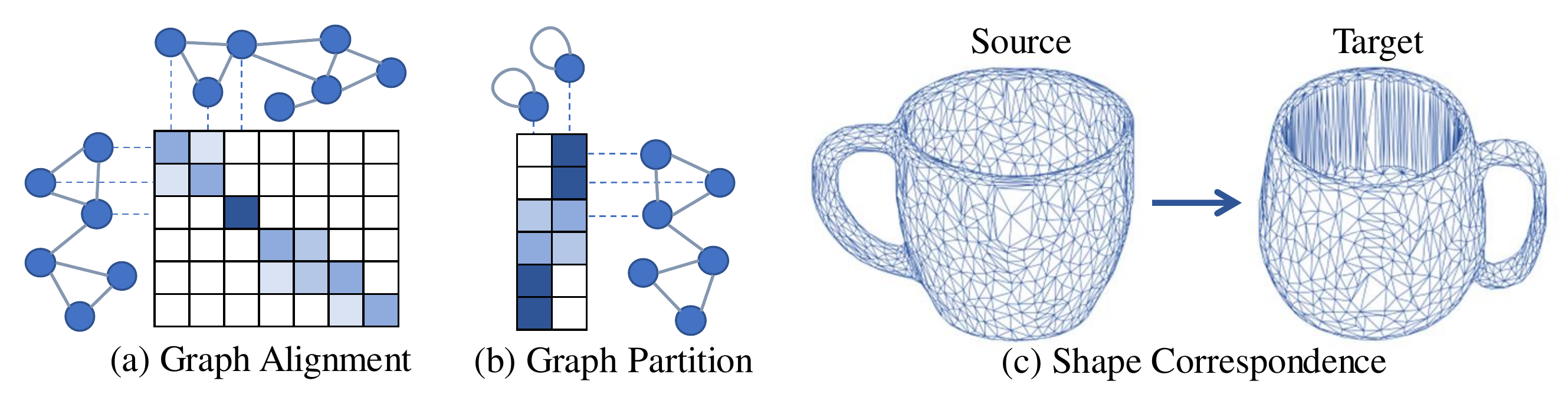}
\vspace{-4mm}
\caption{Graph learning tasks conducted in this paper.}\label{fig:exp_intro}
\vspace{-2mm}
\end{figure}

To bridge the above theoretical-computational gap, we develop provably efficient iterative methods for the GW problem. 
Specifically, we propose two theoretically solid algorithms tailored to different graph learning tasks. 
Our first algorithm is the Bregman Alternating Projected Gradient (\textbf{BAPG}), which is also the first provable single-loop algorithm in the GW literature. The computational stumbling block here is the Birkhoff polytope constraint (i.e., polytope of doubly stochastic matrices) for the coupling matrix, as the associated Bregman projection is either computationally expensive or practically sensitive to hyperparameters. To address this issue, 
we are inspired to decouple the Birkhoff polytope as simplex constraints for rows and columns separately, and then execute the projected gradient descent in an alternating fashion. 
By leveraging the closed-form Bregman projection of the simplex constraint, BAPG only involves the matrix-vector/matrix-matrix multiplications and element-wise matrix operations at each iteration. In short, the proposed single-loop algorithm (BAPG) enjoys a variety of convenient properties: GPU-friendly implementation,  robustness concerning the step size (i.e., the only hyper-parameter), and low memory cost.   Nonetheless, the iterates generated by BAPG do not necessarily satisfy the Birkhoff polytope constraint and can only reach a critical point of the original GW problem asymptotically. To complement BAPG and avoid such a drawback, we revisit BPG from a fresh perspective and introduce our second algorithm --- hybrid Bregman Projected Gradient (\textbf{hBPG}). In particular, we apply eBPG to get a good initial point and then use BPG to get to a critical point. Taking advantage of both eBPG and BPG, the resulting hBPG achieves a great balance between accuracy and efficiency. 

Next, we investigate the convergence behavior of the proposed BAPG and hBPG. A key fact here is that the GW problem is a symmetric nonconvex quadratic program with Birkhoff polytope constraints. By fully exploring this structure, it is interesting to note that the GW problem satisfies the  Luo-Tseng error bound condition~\citep{luo1992error}, which plays a vital role in our analysis. We first quantify the approximation bound for the fixed-point set of BAPG explicitly and the subsequent convergence result follows. Moreover, we prove the local linear convergence result of BPG and hBPG under the established local error bound property.

As a result of the developed theoretical results and algorithm characteristics, we are able to provide novel insights to help users to select the most suitable algorithm. 
In fact, BAPG is the best fit for the graph alignment and partition tasks (see Fig. \ref{fig:exp_intro} (a) and (b) for details), which can sacrifice some feasibility to gain both matching accuracy (i.e., performance measure) and computational efficiency. By contrast, hBPG is more suitable for the shape matching task (i.e.,  Fig. \ref{fig:exp_intro} (c)), where one of the quality metrics is the sharpness of the matching coupling. Hence, BAPG would be the sub-optimal choice due to its infeasibility issue. 
In terms of wall-clock time and modeling performance, extensive experiment results have shown that our methods consistently achieve superior performance on all graph learning tasks.
It is also well worth noting that all theoretical insights have been well-supported by our experiment results. To sum up, this paper opens up an exciting avenue for realizing the benefits of GW in graph learning analysis.

\section{Proposed Algorithms}
\label{sec:method}

In this section, we first formulate the GW distance as a nonconvex quadratic problem with Birkhoff polytope constraints. Then, we introduce two proposed algorithms — BAPG and
hBPG, and further analyze their computational properties
and applicable scenarios.

\subsection{Problem Setup}

The Gromov-Wasserstein (GW) distance was originally proposed in \citep{memoli2011gromov,memoli2014gromov,peyre2019computational} for quantifying the distance between two probability measures supported on heterogeneous metric spaces. More precisely:
\begin{definition}[GW distance]
    \label{defi:gw}
    Suppose that we are given two unregistered compact metric spaces $(\mathcal{X},d_X)$, $(\mathcal{Y},d_Y)$ accompanied with Borel probability measures $\mu,\nu$ respectively. The GW distance between $\mu$ and $\nu$ is defined as 
    \[
    \inf_{\pi \in \Pi(\mu,\nu)}  \iint |d_X(x,x')-d_Y(y,y')|^2 d\pi(x,y)d\pi(x',y'),
    \]
    where $\Pi(\mu,\nu)$ is the set of all probability measures on $\mathcal{X}\times\mathcal{Y}$ with $\mu$ and $\nu$ as marginals. 
 \end{definition}

Intuitively, the GW distance is trying to preserve the isometric structure between two probability measures under optimal derivation. If a map pairs $x\rightarrow y$ and $x'\rightarrow y'$, then the distance between $x$ and $x'$ is supposed to be close to the distance between $y$ and $y'$. In view of these nice properties, the GW distance acts as a powerful modeling tool in structural data analysis, especially in graph learning; see, e.g., \citep{vayer2019sliced,xu2019gromov,xu2019scalable,solomon2016entropic,peyre2016gromov} and the references therein.

To start with our algorithmic developments, we consider the discrete case for simplicity and practicality,  where $\mu$ and $\nu$ are two empirical distributions, i.e., $\mu = \sum_{i=1}^n \mu_i \delta_{x_i}$ and  $\nu = \sum_{j=1}^m \nu_j \delta_{y_j}$. Then, the GW-distance admits the following reformulation:
	\begin{equation}
	\label{eq:gw_qua}
	\begin{aligned}
	 & \min_{\pi \in \mathbb{R}^{n \times m}} -\text{Tr}(D_X\pi D_Y\pi^T)\\
	& \quad \,\text{s.t.} \quad\pi \mathbf{1}_m = \mu, \, \pi^T\mathbf{1}_n = \nu, \,  \pi \ge 0, 
	\end{aligned}
	\end{equation}
where $D_X$ and $D_Y$ are two symmetric distance matrices. 

\subsection{Bregman Alternating Projected Gradient (BAPG)}
Next, we present the proposed Bregman alternating projected gradient (BAPG) method, which is the first provable single-loop algorithm tailored to the GW distance computation. It is easy to observe that Problem \eqref{eq:gw_qua} is a nonconvex quadratic program with polytope constraints. Naturally, if we want to invoke  (Bregman) projected gradient descent type algorithms to address \eqref{eq:gw_qua}, the key difficulty lies in the Birkhoff constraint with respect to the coupling matrix $\pi$. In fact, 
all existing theoretically sound algorithms rely on an inner iterative algorithm to tackle the regularized optimal transport problem at each iteration, such as Sinkhorn \citep{cuturi2013sinkhorn} or semi-smooth Newton method~\citep{li2020efficient}. Arguably, one of the main drawbacks of the double-loop scheme is its computational burden, which is not GPU friendly. 
To circumvent this drawback, we are motivated to handle the row and columns constraints separately (i.e., $C_1=\{\pi\ge 0: \pi \mathbf{1}_m = \mu\}$ and $C_2=\{\pi\ge 0: \pi^T \mathbf{1}_n = \nu\}$).  The crux of BPAG is to take the alternating projected gradient descent step between $C_1$ and $C_2$. Moreover, by fully detecting the hidden structures, we are able to further benefit from the fact that the Bregman projection with respect to the negative entropy for the simplex constraint (e.g., $C_1$ and $C_2$) can be extremely efficient computed~\citep{krichene2015efficient}. The above considerations lead to the closed-form updates in each iteration of BAPG:
\begin{mytcb}[title= BAPG]
\begin{equation}
\label{eq:bapg_update_c}
\begin{aligned}
& \pi \leftarrow \pi \odot \exp({D_X \pi D_Y}/{\rho}), \quad \pi \leftarrow  \text{diag}(\mu./\pi\mathbf{1}_m) \pi, \\
&\pi \leftarrow \pi \odot \exp({D_X \pi D_Y}/{\rho}), \quad 
\pi \leftarrow\pi  \text{diag}(\nu./{\pi}^T\mathbf{1}_n),
\end{aligned}
\end{equation}
\end{mytcb}
where $\rho$ is the step size and $\odot$
denotes element-wise (Hadamard) matrix multiplication. BAPG enjoys several nice properties that are extremely attractive for medium or large-scale graph learning tasks. First, BAPG is a single-loop algorithm that only involves matrix-vector/matrix-matrix multiplications and element-wise matrix operations. All these operations are GPU-friendly. Second, different from the entropic regularization parameter in eBPG, BAPG is more robust to the step size $\rho$ in terms of solution performance. Third, BAPG only involves one memory operation of a large matrix with size $nm$. Notably, even for large-scale optimal transport problems, the main bottleneck is not floating-point computations, but rather time-consuming memory operations~\citep{mai2021fast}.

\paragraph{Theoretical Insights for BAPG}
Now, we start to give the theoretical intuition of why BAPG will work well in practice.
For simplicity, we consider the compact form for illustration:
\begin{equation}
\label{eq:gw_compact}
    \min_{\pi} f(\pi)+g_1(\pi)+g_2(\pi). 
\end{equation}
Here, $f(\pi) =-\text{Tr}(D_X\pi D_Y\pi^T)$ is a nonconvex quadratic function;  $g_1(\pi)=\mathbb{I}_{\{\pi \in C_1\}}$ and 
$g_2(\pi)= \mathbb{I}_{\{\pi \in C_2\}}$ are two indicator functions over closed convex polyhedral sets.
 To better understand the alternating projected scheme developed in \eqref{eq:bapg_update_c},
we adopt the operator splitting strategy to reformulate \eqref{eq:gw_compact} as
\begin{equation}
\label{eq:gw_bilinear}
\begin{aligned}
	 & \min_{ \pi = w}\, f(\pi,w) +g_1(\pi)+g_2(w)\\
	\end{aligned}
	\end{equation}
where  $f(\pi,w) = -\text{Tr}(D_X\pi D_Yw^T)$.  Then, the BAPG update \eqref{eq:bapg_update_c} can be interpreted as processing the alternating minimization scheme on the constructed penalized function, i.e.,
\begin{align*}
& F_\rho(\pi,w) = f(\pi,w)+g_1(\pi)+g_2(w)+\rho D_h(\pi,w).
\end{align*}
Here, $D_h(\cdot,\cdot)$ is the so-called Bregman divergence, i.e.,
$
D_{h}(x, y):=h(x)-h(y)-\langle\nabla h(y), x-y\rangle,
$
where $h(\cdot)$ is the Legendre function, e.g., $\tfrac{1}{2}\|x\|^2$, relative entropy $x\log x$, etc.
For the $k$-th iteration, the BAPG update takes the form
\begin{equation}
\label{eq:bapg_general}
\begin{aligned}
\pi^{k+1} & = \mathop{\arg\min}_{\pi \in C_1}\left\{ f(\pi,w^k) + \rho D_h(\pi, w^k)\right\}, \\
w^{k+1} & = \mathop{\arg\min}_{w\in C_2}\left\{f(\pi^{k+1},w) + \rho D_h(w,\pi^{k+1})\right\}.
\end{aligned}
\end{equation}
When $h$ is the relative entropy and we use the same variable to update $\pi$ and $w$,  \eqref{eq:bapg_update_c} can be derived from \eqref{eq:bapg_general}.

It is worth noting that the way to invoke the operator splitting strategy is not for the algorithm design as usual but instead to provide an intuitive way to explain why BAPG will work. 
As we shall see later, it plays an important role in deriving the approximation bound of the fix-point set of BAPG instead. Moreover,
similar to the quadratic penalty method \citep{nocedal2006numerical}, BAPG is an infeasible method that can only converge to a critical point of \eqref{eq:gw_qua} in an asymptotic sense. In other words, if we choose the parameter $\rho$ as a constant, there is always an infeasibility gap. Fortunately, BAPG is already able to achieve the desired performance for some graph learning tasks, especially those that care more about implementation efficiency and matching accuracy. Recently, \citep{vincent2021semi} has proposed a relaxed version of GW distance for the graph partition task, which further corroborates our investigation that it is acceptable and promising to sacrifice some feasibility to gain other benefits. 
Moreover, \cite{sejourne2021unbalanced} has also proposed a closely related marginal relaxation. That is, we make $\pi = w$ and $F_\rho(\pi,\pi)$ is the objective introduced in \citep{sejourne2021unbalanced}. Unfortunately, they did not develop any provable algorithms for the unbalanced GW distance. 
In Sections \ref{sec:graph_align} and \ref{sec:graph_partition}, extensive experiments have been conducted to demonstrate that BAPG has superior performance compared with other existing baselines on graph alignment and partition tasks.

\subsection{Hybrid Bregman Projected Gradient (hBPG)}
To remedy the infeasible issue of BAPG, we revisit the  Bregman proximal gradient descent (BPG) method, which is a feasible method for addressing the original problem \eqref{eq:gw_qua} exactly. Such an approach has already been well-explored in some early works \citep{xu2019gromov}. 
For the $k$-th iteration, BPG takes the form 
\begin{equation}
\label{eq:bpg_update}
\pi^{k+1} = \mathop{\arg\min}_{\pi\in C_1\cap C_2} \left\{ \nabla f(\pi^k)^T \pi + \frac{1}{t_k}D_h(\pi,\pi^k) \right\}, 
\end{equation}
where $t_k$ is the chosen step size. The core difficulty here is the need for an inner solver to tackle \eqref{eq:bpg_update} efficiently. As it turns out, without the entropic regularizer, BPG suffers from the numerical instability issue. It is difficult for the inner solver to achieve the desired accuracy so as to guarantee convergence. Although \cite{xu2019gromov} has provided a  subsequential convergence result for BPG, it is too weak to guide the user in which scenarios BPG will potentially enjoy notable advantages.
Such a state of affairs greatly limits its applicability. As we shall see later in Section \ref{sec:bpg},  we explicitly analyze this stability issue  and explain the reason rigorously.
Different from graph alignment and partition tasks studied in \cite{xu2019gromov,xu2019scalable}, BPG-type methods are more attractive for applications that require a sharp matching map (such as shape correspondence), since the approximation (infeasibility) gap will dramatically affect the modeling performance. 

We further exploit its local linear convergence property. Such a property guides us to take full advantage of both methods --- BPG and eBPG. A natural idea is to apply eBPG to get a good initial point and then use BPG to reach critical points.  It is reasonable to infer that the resulting hybrid method (denoted by hBPG) will achieve a trade-off between accuracy and efficiency.
Comprehensive experiments conducted in Section \ref{sec:shape} corroborate our theoretical insights.

At last, we summarize all the aforementioned algorithms in Table \ref{tab:algo}, aiming to provide a theoretically-supported guideline for readers on how to select the best-fit method based on their task properties. 
\begin{table}
\caption{Comparison of various algorithms. ``Exactly solve'' means that  algorithms can reach a critical point of the original GW problem instead of the approximation problem.  }
\centering
\begin{tabular}{@{}l|c|c|c@{}}
\toprule
& Single-loop? & Provable? & Exactly solve \eqref{eq:gw_qua}?\,\\
      \midrule
\, eBPG~\citep{solomon2016entropic} &  \textcolor{blue!50!black}{No} & \textcolor{red!50!black}{Yes} &  \textcolor{blue!50!black}{No}\\
\, BPG~\citep{xu2019gromov} & \textcolor{blue!50!black}{No} & \textcolor{red!50!black}{Yes}&\textcolor{red!50!black}{Yes}\\
\, BPG-S~\citep{xu2019scalable} & \textcolor{red!50!black}{Yes} &  \textcolor{blue!50!black}{No}&  \textcolor{blue!50!black}{No}\\
\, FW~\citep{titouan2019optimal}&\textcolor{blue!50!black}{No} & \textcolor{red!50!black}{Yes}&\textcolor{red!50!black}{Yes} \\
\, BAPG (\textbf{our method}) & \textcolor{red!50!black}{Yes}& \textcolor{red!50!black}{Yes}& \textcolor{blue!50!black}{No}\\
      \bottomrule
\end{tabular}
\label{tab:algo}
\end{table}
\section{Theoretical Results}
\label{sec:conv}
In this section, we present all theoretical results conducted in this paper, including the approximation bound of the fix-point set of BAPG and its convergence analysis.
At the heart of our analysis is that  the following regularity condition holds for the GW problem~\eqref{eq:gw_qua}. 
\begin{proposition}[Luo-Tseng Error Bound Condition for \eqref{eq:gw_qua}] 
\label{prop:erb}
There exist scalars $\epsilon>0$ and $\tau >0$ such that 
\begin{equation}
\label{eq:erb}
\dist2(\pi,\mathcal{X}) \leq \tau \left\|\pi - \proj_{C_1\cap C_2}(\pi+D_X \pi D_Y)\right\|,
\end{equation}
whenever $\|\pi - \proj_{C_1\cap C_2}(\pi+D_X \pi D_Y)\|\leq \epsilon$, where $\mathcal{X}$ is the critical point set of \eqref{eq:gw_compact} defined by 
\begin{equation}
\label{eq:gw_critical}
    \mathcal{X} = \{\pi \in C_1\cap C_2: 0\in \nabla f(\pi)+\mathcal{N}_{C_1}(\pi)  +\mathcal{N}_{C_2}(\pi)\}
\end{equation}
and $\mathcal{N}_{C}(\pi)$ is the normal cone to $C$ at $\pi$. 
\end{proposition}
\begin{proof}
$D_X$ and $D_Y$ are two symmetric matrices and $C_1\cap C_2$ is a convex polyhedral set. By invoking Theorem 2.3 in \citep{luo1992error},  the Luo-Tseng local error bound condition \eqref{eq:erb} just holds for the feasible set $C_1 \cap C_2$. That is, 
\begin{equation}
\label{eq:erb_a}
\dist2(\pi,\mathcal{X}) \leq \tau \left\|\pi - \proj_{C_1\cap C_2}(\pi+D_X \pi D_Y)\right\|,
\end{equation}
where $\pi \in C_1\cap C_2$.
Then, we aim at extending \eqref{eq:erb_a} to the whole space. Define $\tilde{\pi} = \proj_{C_1 \cap C_2} (\pi)$, we have 
\begin{equation}\label{ineq:1}
    \dist2(\pi,\mathcal{X}) \leq  \|\pi - \tilde{\pi}\| + d(\tilde{\pi}, \mathcal{X}) \leq  \|\pi - \tilde{\pi}\| + \tau \|\tilde{\pi} -\proj_{C_1\cap C_2}(\tilde{\pi}+D_X \tilde{\pi} D_Y)\|
\end{equation}
where the first inequality holds because of the triangle inequality, and the second one holds because of the fact $\tilde{\pi} \in C_1 \cap C_2$ and \eqref{eq:erb_a}. Apply the triangle inequality again, we have
\begin{equation}\label{ineq:2}
\begin{aligned}
    &\|\tilde{\pi} -\proj_{C_1\cap C_2}(\tilde{\pi}+D_X \tilde{\pi} D_Y)\| \\ \leq & \|\tilde{\pi} -\proj_{C_1\cap C_2}(\pi+D_X \pi D_Y)\| +  \|\proj_{C_1\cap C_2}(\pi+D_X \pi D_Y) - \proj_{C_1\cap C_2}(\tilde{\pi}+D_X \tilde{\pi} D_Y)\|
    \end{aligned}
\end{equation}
Since $\tilde{\pi} = \proj_{C_1 \cap C_2} (\pi)$, owing to the fact that the projection operator onto a convex set $\proj_{C_1 \cap C_2}(\cdot)$ is non-expensive, i.e., $\|\proj_{C_1 \cap C_2}(x) - \proj_{C_1 \cap C_2}(y)\| \leq \|x - y\|$ holds for any $x$ and $y$, we have 
\begin{equation}\label{ineq:3}
\begin{aligned}
     \|\tilde{\pi} -\proj_{C_1\cap C_2}(\pi+D_X \pi D_Y)\| & =  \|\proj_{C_1 \cap C_2} (\pi) -\proj_{C_1\cap C_2}( \proj_{C_1\cap C_2}(\pi+D_X \pi D_Y))\| \\
     &\leq \|\pi -\proj_{C_1\cap C_2}(\pi+D_X \pi D_Y)\|,
    \end{aligned}
\end{equation}
and 
\begin{equation}\label{ineq:4}
\begin{aligned}
    \|\proj_{C_1\cap C_2}(\pi+D_X \pi D_Y) - \proj_{C_1\cap C_2}(\tilde{\pi}+D_X \tilde{\pi} D_Y)\| & \leq \| (\pi+D_X \pi D_Y) -  (\tilde{\pi}+D_X \tilde{\pi} D_Y)\| \\ & \leq (\sigma_{\text{max}}(D_X)\sigma_{\text{max}}(D_Y)+1)\|\pi- \tilde{\pi}\|
         \end{aligned}
\end{equation}
where $\sigma_{\text{max}}(D_X)$ and $\sigma_{\text{max}}(D_Y)$ denote the maximum singular value of $D_X$ and $D_Y$, respectively. By substituting \eqref{ineq:2} and \eqref{ineq:4} into \eqref{ineq:1}, we get
\begin{equation}\label{ineq:5}
\begin{aligned}
    \dist2(\pi,\mathcal{X}) & \leq
    ( \tau\sigma_{\text{max}}(D_X)\sigma_{\text{max}}(D_Y)+\tau+1)\|\pi - \tilde{\pi}\|+ \tau\| \pi -\proj_{C_1\cap C_2}(\pi+D_X \pi D_Y) \| \\
    & \leq ( \tau\sigma_{\text{max}}(D_X)\sigma_{\text{max}}(D_Y)+2 \tau+1) \| \pi -\proj_{C_1\cap C_2}(\pi+D_X \pi D_Y) \|
    \end{aligned}
\end{equation}
where the last inequality holds because of the fact that
\begin{align*}
    \|\pi - \tilde{\pi}\| = \|\pi - \proj_{C_1\cap C_2}(\pi)\| = \min_{x \in C_1\cap C_2} \|\pi - x\| \leq \| \pi -\proj_{C_1\cap C_2}(\pi+D_X \pi D_Y) \|.
\end{align*}
By letting $\tau = \tau\sigma_{\text{max}}(D_X)\sigma_{\text{max}}(D_Y)+2\tau+1$, we get the desired result. 
\end{proof}
As the GW problem is a nonconvex quadratic program with polytope constraint, we can invoke Theorem 2.3 in \citep{luo1992error} to conclude that the  error bound condition \eqref{eq:erb} holds on the whole feasible set $C_1 \cap C_2$. Proposition \ref{prop:erb} extends \eqref{eq:erb} to the whole space $\mathbb{R}^{n\times m}$. This regularity condition is trying to bound the distance of any coupling matrix to the critical point set  of the GW problem by its optimality residual, which is characterized by the difference for one step projected gradient descent.
It turns out that this error bound condition  plays an important role in quantifying the approximation bound for the fixed points set of BAPG explicitly. 
\subsection{Approximation Bound for the Fix-Point Set of BAPG}
To start, we present one key lemma that shall be used in studying the approximation bound of BAPG. 
\begin{lemma}
\label{lem:lemma_proj}
Let $C_{1}$ and $C_{2}$ be convex polyhedral sets. There exists a constant $M>0$ such that
\begin{equation}
\label{eq:lemma_proj}
\begin{aligned}
&\left\|\proj_{C_{1}}(x)+\proj_{C_{2}}(y)- 2 \proj_{C_1\cap C_2}\left(\frac{x+y}{2}\right)\right\| 
\leqslant  M\left\|\proj_{C_{1}}(x)-\proj_{C_{2}}(y)\right\|, \forall x\in C_1, y\in C_2.
\end{aligned}
\end{equation}
\end{lemma}

\begin{proof}
we first convert the left-hand side of \eqref{eq:lemma_proj} to
 \begin{align*}
    \ & \left\|\proj_{C_{1}}(x)-\proj_{C_1\cap C_2}\left(\frac{x+y}{2}\right) +\proj_{C_{2}}(y)-  \proj_{C_1\cap C_2}\left(\frac{x+y}{2}\right)\right\| \\
    \leq \ & \sqrt{2}\left\| \left(\proj_{C_1\cap C_2}\left(\frac{x+y}{2}\right),\proj_{C_1\cap C_2}\left(\frac{x+y}{2}\right)\right)-(\proj_{C_{1}}(x),\proj_{C_{2}}(y)) \right\|. 
\end{align*}
The core proof idea follows essentially from the observation that the inequality can be regarded as the stability of the optimal solution for a linear-quadratic problem, i.e., 
\begin{equation}
\label{eq:proj_cons}
(p(r),q(r)) = 
\begin{aligned}
 & \mathop{\arg\min}_{p,q} \frac{1}{2}\|x-p\|^2 +\frac{1}{2}\|y-q\|^2\\
 & \,\, \quad \text{s.t.} \, \, \quad  p-q= r, \\
 & \quad \quad \quad \quad  p \in C_1, q\in C_2. 
\end{aligned}
\end{equation}
When $r = 0$, the pair $(p(0),q(0))$ satisfies $p(0)=q(0)=\proj_{C_1\cap C_2}\left(\frac{x+y}{2}\right)$. Moreover, the parameter $r$ itself can be viewed as the perturbation quantity, which is indeed the right-hand side of \eqref{eq:lemma_proj}. By invoking Theorem 4.1 in \citep{zhang2020global}, we can bound the distance between two optimal solutions by the perturbation quantity $r$, i.e., 
\begin{align*}
     \|(p(0),q(0))-(p(r),q(r))\| \leq M \|r\|.
\end{align*}

\begin{remark}
For any vector $r$, let $(p(r), q(r))$ be the optimal solution to \eqref{eq:proj_cons}. On the one hand, choosing $r=0$, it is easy to see that $p(0)= q(0) = \proj_{C_1 \cap C_2} (\frac{x+y}{2})$. On the other hand, by choosing $r = \proj_{C_1}(x) - \proj_{C_2}(y)$, it is easy to see that $(p(r), q(r)) = (\proj_{C_1}(x), \proj_{C_2}(y))$. Since Theorem 4.1 in \citep{zhang2020global} implies that for any $r$, we have
	    \begin{align*}
	        \|(p(0), q(0)) - (p(r), q(r))\|_2 \leq M \|r\|_2
	    \end{align*}
	    where $M$ only depends on $C_1$ and $C_2$. By choosing $r = \proj_{C_1}(x) - \proj_{C_2}(y)$, we have
	    \begin{align*}
	        \|(\proj_{C_1 \cap C_2} (\frac{x+y}{2}), \proj_{C_1 \cap C_2} (\frac{x+y}{2})) - (\proj_{C_1}(x), \proj_{C_2}(y)) \leq M \|\proj_{C_1}(x) - \proj_{C_2}(y)\|_2.
	    \end{align*}
	\end{remark}
\end{proof}
Equipped with Lemma \ref{lem:lemma_proj} and Proposition \ref{prop:erb}, it is not hard to obtain the following approximation result. 
\begin{proposition} [Approximation Bound of the Fix-point Set of BAPG]
\label{prop:bapg_approx}

The point \\$(\pi^\star,w^\star)$ belongs to the fixed-point set $\mathcal{X}_{\text{BAPG}}$ of BAPG if it satisfies 
\begin{equation}
\label{eq:x_bapg}
\begin{aligned}
& \nabla f(w^\star) + \rho (\nabla h(\pi^\star)-\nabla h(w^\star)) + p =0,  \\ 
& \nabla f(\pi^\star) + \rho (\nabla h(w^\star)-\nabla h(\pi^\star)) + q =0, 
\end{aligned}
\end{equation} 
where $p \in \mathcal{N}_{C_1}(\pi^\star)$ and $q \in \mathcal{N}_{C_2}(w^\star)$. Then, 
the infeasibility error satisfies $\|\pi^\star-w^\star\| \leq  \frac{\tau_1}{\rho}$ and the gap between $\mathcal{X}_{\text{BAPG}}$ and $\mathcal{X}$ satisfies
\[
\dist2\left(\frac{\pi^\star+w^\star}{2},\mathcal{X}\right) \leq \frac{\tau_2}{\rho},
\]
where $\tau_1$ and $\tau_2$ are two constants. 
\end{proposition}
\begin{proof}
Define $\hat{\pi} = \proj_{C_1\cap C_2}(\pi^\star)$, we first want to argue the following inequality holds, 
\begin{align*}
    \|\hat{\pi}-\pi^\star\| + \|\hat{\pi}-w^\star\| \leq (2\kappa+1)\|\pi^\star-w^\star\|. 
\end{align*}
As the bounded linear regularity condition is satisfied for the polyhedral constraint, we have 
\begin{equation}
\label{eq:blr_inf}
\begin{aligned}
    \|\hat{\pi}-\pi^\star\| +  \|\hat{\pi}-w^\star\| & \leq 2\|\hat{\pi}-\pi^\star\| + \|\pi^\star-w^\star\|\\
    & = 2\dist2(\pi^\star,C_1\cap C_2)+\|\pi^\star-w^\star\| \\
    & \leq (2\kappa +1)\|\pi^\star-w^\star\|. 
\end{aligned}
\end{equation}
Based on the stationary points defined in \eqref{eq:x_bapg}, we have, 
\[
\nabla f(w^\star)^T(\hat{\pi}-\pi^\star) + \rho (\nabla h(\pi^\star)-\nabla h(w^\star))^T(\hat{\pi}-\pi^\star)+ p^T(\hat{\pi}-\pi^\star) =0, p \in \mathcal{N}_{C_1}(\pi^\star),
\]
\[
\nabla f(\pi^\star)^T(\hat{\pi}-w^\star) + \rho (\nabla h(w^\star)-\nabla h(\pi^\star))^T(\hat{\pi}-w^\star)  + q^T(\hat{\pi}-w^\star)  =0, q \in \mathcal{N}_{C_2}(w^\star).
\]
Summing up the above two equations, 
\begin{align*}
& \nabla f(w^\star)^T(\hat{\pi}-\pi^\star) + \nabla f(\pi^\star)^T(\hat{\pi}-w^\star) + \rho (\nabla h(\pi^\star)-\nabla h(w^\star))^T(w^\star-\pi^\star)+p^T(\hat{\pi}-\pi^\star)+q^T(\hat{\pi}-w^\star) =0 \\
\stackrel{(\clubsuit)}{\Rightarrow} & 
\nabla f(w^\star)^T(\hat{\pi}-\pi^\star) + \nabla f(\pi^\star)^T(\hat{\pi}-w^\star) - \rho(D_h(\pi^\star,w^\star)+D_h(w^\star,\pi^\star))+p^T(\hat{\pi}-\pi^\star)+q^T(\hat{\pi}-w^\star) =0 \\
\stackrel{(\spadesuit)}{\Rightarrow} & \nabla f(w^\star)^T(\hat{\pi}-\pi^\star) + \nabla f(\pi^\star)^T(\hat{\pi}-w^\star) - \rho(D_h(\pi^\star,w^\star)+D_h(w^\star,\pi^\star)) \ge 0,  
\end{align*}
where $(\clubsuit)$ holds as $(\nabla h(x)-\nabla h(y))^T(x-y) = D_h(x,y)+D_h(y,x)$ and $(\spadesuit)$ follows from the property of the normal cone. For instance, since $q \in \mathcal{N}_{C_1}(\pi^\star)$, we have $q^T(\hat{\pi}-\pi^\star)\leq 0$ (i.e., $\hat{\pi} \in C_1$). Therefore, 
\begin{align*}
    \rho(D_h(\pi^\star,w^\star)+D_h(w^\star,\pi^\star)) & \leq  \nabla f(w^\star)^T(\hat{\pi}-\pi^\star) + \nabla f(\pi^\star)^T(\hat{\pi}-w^\star)\\
    & \leq \|\nabla f(w^\star)\|\|\hat{\pi}-\pi^\star\|+\|\nabla f(\pi^\star)\|\|\hat{\pi}-w^\star\| \\
    & \leq L_f(\|\hat{\pi}-\pi^\star\|+\|\|\hat{\pi}-w^\star\|)\\
    & \stackrel{\eqref{eq:blr_inf}}{\leq} (2\kappa+1)L_f\|\pi^\star-w^\star\|.
\end{align*}
 The next to last inequality holds as $f(\cdot)$ is a quadratic function and the effective domain $C_1 \cap C_2$ is bounded. Thus, the norm of its gradient will naturally have a constant upper bound, i.e., $L_f = \sigma_{\text{max}}(D_X)\sigma_{\text{max}}(D_Y)$. As $h$ is a $\sigma$-strongly convex, we have 
\[D_h(\pi^\star,w^\star) \ge \frac{\sigma}{2}\|\pi^\star-w^\star\|^2.\]
Together with this property, we can quantify the infeasibility error $\|\pi^\star-w^\star\|$,
\begin{equation}
    \label{eq:infeas_BAPG}
     \|\pi^\star-w^\star\| \leq \frac{(2\kappa+1)L_f}{\delta\rho}.
\end{equation}
When $\rho \rightarrow +\infty$, it is easy to observe that the infeasibility error term $\|\pi^\star-w^\star\|$ will shrink to zero. More importantly, if $\pi^\star=w^\star$, then $X_{\text{BAPG}}$ will be identical to $\mathcal{X}$. 
Next, we target at quantifying the approximation gap between the fixed-point set of BAPG and the critical point set of the original problem \eqref{eq:gw_qua}. Upon \eqref{eq:x_bapg}, we have 
\[
\nabla f \left(\frac{\pi^\star+w^\star}{2}\right)+\frac{p+q}{2} = 0,  
\]
where $\nabla f(\cdot)$ is a linear operator. By applying the Luo-Tseng local error bound condition of \eqref{eq:gw_qua}, i.e., Proposition \ref{prop:erb}, we have
\begin{equation*}
\begin{aligned}
     \dist2\left(\frac{\pi^\star+w^\star}{2}, \mathcal{X}\right)  & \leq \tau \left\|\frac{\pi^\star+w^\star}{2} - \proj_{C_1\cap C_2}\left(\frac{\pi^\star+w^\star}{2}-\nabla f \left(\frac{\pi^\star+w^\star}{2}\right)\right)\right\| \\
     & =  \tau \left\|\frac{\pi^\star+w^\star}{2} - \proj_{C_1\cap C_2}\left(\frac{\pi^\star+p+w^\star+q}{2}\right)\right\| \\
     & \stackrel{(\clubsuit)}{=} \frac{\tau}{2}\left\| \proj_{C_1}(\pi^\star+p) + \proj_{C_1}(w^\star+q) - 2\proj_{C_1\cap C_2}\left(\frac{\pi^\star+p+w^\star+q}{2}\right) \right\| \\
     & \stackrel{(\spadesuit)}{\leq} \frac{M\tau}{2}\|\pi^\star- w^\star\|,
\end{aligned}
\end{equation*}
where $(\clubsuit)$ holds due to the normal cone property, that is, $\proj_{C}(x+z) = x$ for any $x \in C$ and $z \in \mathcal{N}_C(x)$, and $(\spadesuit)$ follows from Lemma \ref{lem:lemma_proj}. Incorporating with \eqref{eq:infeas_BAPG}, the approximation bound for BAPG has been characterized quantitatively, i.e., 
\[
\dist2\left(\frac{\pi^\star+w^\star}{2}, \mathcal{X}\right)\leq  \frac{(2\kappa+1)L_fM}{2\delta\rho}. 
\]
\end{proof}
\begin{remark}
If $\pi^\star = w^\star$, then $\mathcal{X}_{\text{BAPG}}$ is identical to $\mathcal{X}$ and BAPG can reach a critical point of the GW problem \eqref{eq:gw_qua}. Proposition \ref{prop:bapg_approx} indicates that as $\rho \rightarrow +\infty$, the infeasibility error term $\|\pi^\star-w^\star\|$ shrinks to zero and thus BAPG converges to a critical point of \eqref{eq:gw_qua} in an asymptotic way. Furthermore, it explicitly quantifies the approximation gap when we select the parameter $\rho$ as a constant. The explicit form of $\tau_1$ and $\tau_2$ only depend on the problem itself, including $\sigma_{\textnormal{max}}(D_X)\sigma_{\textnormal{max}}(D_Y)$, the constant for the Luo-Tseng error bound condition in Proposition \ref{prop:erb} and so on.
\end{remark}
\subsection{Convergence Analysis of BAPG}
A further natural question is whether BAPG will converge or not. We answer the question in the affirmative. Specifically, we show that under several canonical assumptions, any limit point of BAPG belongs to $\mathcal{X}_{\text{BAPG}}$. 
Towards that end, let us first establish the sufficient decrease property of BAPG based on the potential function $F_\rho(\cdot)$,
\begin{proposition}
\label{prop:basic_bapg}
Let $\{(\pi^k,w^k)\}_{k\ge 0}$ be the sequence generated by BAPG. Suppose that $D_h(\cdot,\cdot)$ is symmetric on the whole sequence. Then, we have
    \begin{equation}
    \label{eq:bapg_suff_decre}
    \begin{aligned}
    & F_\rho(\pi^{k+1},w^{k+1}) - F_\rho(\pi^k ,w^k) \leq -  \rho D_h(\pi^k,\pi^{k+1}) - \rho D_h(w^k,w^{k+1}). 
    \end{aligned}
    \end{equation}
\end{proposition}
\begin{proof}
We first observe from the optimality conditions of main updates, i.e., 
\begin{align}
\label{eq:bapg_update_opt1}
& 0 \in \nabla_\pi f(\pi^{k},w^{k}) + \rho (\nabla h(\pi^{k+1})-\nabla h(w^{k})) + \partial g_1(\pi^{k+1}) \\
\label{eq:bapg_update_opt2}
&  0 \in \nabla_w f(\pi^{k+1},w^{k})+ \rho (\nabla h(w^{k+1})-\nabla h(\pi^{k+1})) +\partial g_2(w^{k+1})
\end{align}
where $g_1(\pi) = I_{\{\pi \in C_1\}}$ and $g_2(w) = I_{ \{w\in C_2\}}$. 
Due to the convexity of $g_1(\cdot)$, it is natural to obtain, 
\begin{align*}
g_1(\pi^{k})-g_1(\pi^{k+1}) & \ge -\langle \nabla_\pi f(\pi^k, w^{k})
+\rho (\nabla h(\pi^{k+1})-\nabla h(w^{k})), \pi^k-\pi^{k+1}\rangle \\
& = -\langle \nabla_\pi f(\pi^k, w^{k}),\pi^k-\pi^{k+1}\rangle +\langle \rho (\nabla h(\pi^{k+1})-\nabla h(w^{k})), \pi^{k+1}-\pi^{k}\rangle
\end{align*}
As $f(\pi,w) = -\text{tr}(D_X\pi D_Y w^T)$ is a bilinear function, we have,
\[
f(\pi^{k},w^k) - f(\pi^{k+1}, w^k)-\langle \nabla_\pi f(\pi^k, w^{k}),\pi^k-\pi^{k+1}\rangle = 0. 
\]
Consequently, we get 
\begin{equation}
\label{eq:BAPG_up1}
    f(\pi^{k},w^k)+g_1(\pi^{k})-f(\pi^{k+1}, w^k)-g_1(\pi^{k+1}) \ge  \rho \langle \nabla h(\pi^{k+1})-\nabla h(w^{k}), \pi^{k+1}-\pi^{k}\rangle. 
\end{equation}
Similarly, based on the $w$-update, we obtain 
\begin{equation}
\label{eq:BAPG_up2}
    f(\pi^{k+1},w^{k})+g_2(w^{k})-f(\pi^{k+1}, w^{k+1})-g_2(w^{k+1}) \ge  \rho \langle \nabla h(w^{k+1})-\nabla h(\pi^{k+1}), w^{k+1}-w^{k}\rangle. 
\end{equation}
Combine with \eqref{eq:BAPG_up1} and \eqref{eq:BAPG_up2}, we obtain
\begin{equation}
\label{eq:bapg_decrease}
\begin{aligned}
    & f(\pi^{k},w^k)+g_1(\pi^{k})+g_2(w^{k})-f(\pi^{k+1}, w^{k+1})-g_1(\pi^{k+1})-g_2(w^{k+1}) \\
    \geq \, \, &  \rho \langle \nabla h(\pi^{k+1})-\nabla h(w^{k}), \pi^{k+1}-\pi^{k}\rangle + \rho \langle \nabla h(w^{k+1})-\nabla h(\pi^{k+1}), w^{k+1}-w^{k}\rangle. 
\end{aligned}
\end{equation}
The right-hand side can be further simplified, 
\begin{align*}
     & \rho \langle \nabla h(\pi^{k+1})-\nabla h(w^{k}), \pi^{k+1}-\pi^{k}\rangle + \rho \langle \nabla h(w^{k+1})-\nabla h(\pi^{k+1}), w^{k+1}-w^{k}\rangle\\
     \stackrel{(\clubsuit)}{=}& \rho (D_h(\pi^k,\pi^{k+1})+D_h(\pi^{k+1},w^k)-D_h(\pi^k,w^k)) +\rho(D_h(w^k,w^{k+1}) + \\
     & D_h(w^{k+1},\pi^{k+1})-D_h(w^k,\pi^{k+1}))\\
     \stackrel{(\spadesuit)}{=} &  \rho D_h(\pi^k,\pi^{k+1}) + \rho D_h(w^k,w^{k+1}) - \rho D_h(\pi^k,w^k)+ \rho D_h(\pi^{k+1},w^{k+1}).
\end{align*}
Here, $(\clubsuit)$ uses the fact that the three-point property of Bregman divergence holds, i.e., 
 For any $y, z \in$ int $\operatorname{dom} h$ and $x \in \operatorname{dom} h$,
$$
D_{h}(x, z)-D_{h}(x, y)-D_{h}(y, z)=\langle\nabla h(y)-\nabla h(z), x-y\rangle.
$$
Moreover, $(\spadesuit)$ holds as $D_h(\cdot, \cdot)$ is symmetric on the whole sequence.
To proceed, this together with \eqref{eq:bapg_decrease} implies
\begin{equation}
\label{eq:suff_bapg}
F_\rho(\pi^{k+1},w^{k+1}) - F_\rho(\pi^k ,w^k) \leq -  \rho D_h(\pi^k,\pi^{k+1}) - \rho D_h(w^k,w^{k+1}). 
\end{equation}

Summing up \eqref{eq:suff_bapg} from $k=0$ to $+\infty$, we obtain 
\[
F_\rho(\pi^{\infty},w^\infty)-F_\rho(\pi^0,w^0) \leq -\kappa_1 \sum_{k=0}^\infty \left(D_h(\pi^k,\pi^{k+1}) + D_h(w^k,w^{k+1}) \right). 
\]
\end{proof}
As the potential function $F_\rho(\cdot,\cdot)$ is coercive and $\{(\pi^{k},w^k)\}_{k \ge 0}$ is a bounded sequence, it means the left-hand side is bounded, which implies
\[
\sum_{k=0}^\infty \left(D_h(\pi^k,\pi^{k+1}) + D_h(w^k,w^{k+1}) \right) < +\infty. 
\]
Both $\{D_h(\pi^k,\pi^{k+1})\}_{k \ge 0}$ and $\{D_h(w^k,w^{k+1})\}_{k \ge 0}$ converge to zero. Thus, the following convergence result holds.
\begin{theorem}[Subsequent Convergence of BAPG]
\label{thm:sub_bapg}
Any limit point of the sequence \\$\{(\pi^{k},w^k)\}_{k \ge 0 }$ generated by BAPG belongs to $\mathcal{X}_{\text{BAPG}}$.
\end{theorem}
\begin{proof}
 Let $(\pi^\infty, w^\infty)$ be a limit point of the sequence $\{(\pi^{k},w^{k})\}_{k \ge 0 }$. Then, there exists a sequence $\{n_k\}_{k\ge0}$ such that $\{(\pi^{n_k},w^{n_k})\}_{k \ge 0 }$ converges to $(\pi^\infty, w^\infty)$. Replacing $k$ by $n_k$ in \eqref{eq:bapg_update_opt1} and \eqref{eq:bapg_update_opt2}, 
 taking limits on both sides as $k\rightarrow\infty$ 
\begin{align*}
& 0 \in \nabla_\pi f(\pi^{\infty},w^{\infty}) + \rho (\nabla h(\pi^{\infty})-\nabla h(w^{\infty})) + \partial g_1(\pi^{\infty}) \\
&  0 \in \nabla_w f(\pi^{\infty},w^{\infty})+ \rho (\nabla h(w^{\infty})-\nabla h(\pi^{\infty})) +\partial g_2(w^{\infty}).
\end{align*}
Based on the fact that $\nabla_\pi f(\pi^{\infty},w^{\infty}) = \nabla f(w^\infty)$ and $\nabla_w f(\pi^{\infty},w^{\infty}) = \nabla f(\pi^\infty)$, it can be easily concluded that $(\pi^\infty, w^\infty) \in \mathcal{X}_{\text{BAPG}}$.
\end{proof}

\begin{remark}
The symmetric assumption in Proposition \ref{prop:basic_bapg} does not hold for the KL divergence  in general. However, if we add a mild condition that the asymmetric ratio is bounded for iterations, similar results can be achieved. Moreover, when $h$ is a quadratic function, we can further obtain the global convergence result under the  Kurdyka-Lojasiewicz analysis framework 
developed in \citep{attouch2010proximal,attouch2013convergence}. 
\end{remark}

More importantly, the above subsequence convergence result can be extended to the global convergence if the Legendre function $h(\cdot)$ is quadratic.  
\begin{theorem}[Global Convergence of BAPG --- Quadratic Case]
The sequence \\ $\{(\pi^{k},w^k)\}_{k \ge 0 }$ converges to a critical point of $F_\rho(\pi, w)$. 
\end{theorem}
\begin{proof}
To invoke the Kurdyka-Lojasiewicz analysis framework 
developed in \citep{attouch2010proximal,attouch2013convergence}, we have to establish two crucial properties --- sufficient decrease and safeguard condition. The first one has already been proven in Proposition \ref{prop:basic_bapg}. Then, we would like to show that there exists a constant $\kappa_2$ such that, 
\[
\dist2(0,   \partial F_\rho(\pi^{k+1},w^{k+1})) \leq \kappa_2 (\|\pi^{k+1}-\pi^k\|+\|w^{k+1}-w^k\|).
\]

At first, noting that 
\begin{equation*}
    \partial F_\rho(\pi^{k+1},w^{k+1}) = \left[
    \begin{aligned}
      & \nabla_\pi f(\pi^{k+1},w^{k+1}) + \rho (\pi^{k+1}-w^{k+1}) + \partial g_1(\pi^{k+1}) \\
      & \nabla_w f(\pi^{k+1},w^{k+1}) + \rho (w^{k+1}-\pi^{k+1}) + \partial g_2(w^{k+1})
    \end{aligned}
    \right].
\end{equation*}
Again, together with the updates \eqref{eq:bapg_update_opt1} and \eqref{eq:bapg_update_opt2}, this implies
 \begin{equation*}
 \begin{aligned}
 \dist2^2(\mathbf{0},\partial F_\rho(\pi^{k+1},w^{k+1})) & \leq \|\nabla_\pi f(\pi^{k+1},w^{k+1})-\nabla_\pi f(\pi^{k},w^{k})\|^2 + \rho \| w^k- w^{k+1}\|^2 \\
 & + \|\nabla_w f(\pi^{k+1},w^{k})-\nabla_w f(\pi^{k+1},w^{k+1})\|^2 \\
 & \leq (\sigma^2_{\text{max}}(D_X)\sigma^2_{\text{max}}(D_Y)+\rho) \| w^k- w^{k+1}\|^2, 
 \end{aligned}
 \end{equation*}
 where $\sigma_{\text{max}}(\cdot)$ denotes the maximum singular value. 
 By letting $\kappa_2 = \sqrt{\sigma^2_{\text{max}}(D_X)\sigma^2_{\text{max}}(D_Y)+\rho}$, we get the desired result.
\end{proof}

To the best of our knowledge, the convergence analysis of alternating projected gradient descent methods has only been given under the convex setting, see \citep{wang2016stochastic,nedic2011random} for details.
In this paper, by heavily exploiting the error
bound condition of the GW problem,  we take the first step and provide a new path to conduct the analysis of alternating projected descent
method for nonconvex problems, which could be of independent interest.

\subsection{Local Linear Convergence of BPG and hBPG}
\label{sec:bpg}
Next, we investigate the convergence behavior of BPG when applied to problem \eqref{eq:gw_qua}.
At the heart of our convergence rate analysis is the Luo-Tseng local error bound condition (cf. Proposition \ref{prop:erb}), which has been demonstrated as a crucial tool to establish the linear convergence rate of first-order algorithms~\citep{zhou2017unified}.  
Recall that a sequence  $\{x^k\}_{k\ge0}$ is said to converge R-linearly (resp. Q-linearly) to  a point $x^\infty$ if there exist  constants $\gamma>0$  and $\eta \in (0,1) $ such that $\|x^k-x^\infty\|\leq \gamma \eta^{k}$  for all $k\ge0$ (resp. if there exists an index $K\ge0$ and a constant $\eta \in (0,1)$ such that $\|x^{k+1}-x^\infty\|/\|x^{k}-x^\infty\|\leq \eta$ for all $k\ge K$).

\begin{theorem}[Local Linear Convergence of BPG]
\label{thm:bpg}
Suppose that in Problem \eqref{eq:gw_qua}, the step size $t_k$ in \eqref{eq:bpg_update} satisfies $0 <\underline{t}\leq t_k<\overline{t}\leq {\sigma}/{L_f}$ for $k \ge 0$ where $\underline{t},\overline{t}$ are given constants and $L_f$ is the gradient Lipschitz constant of $f$. Moreover, suppose that the sequence $\{\pi^k\}_{k\ge0}$ has a element-wise lower bound, i.e., $\pi_k \ge \epsilon>0$, the sequence of solutions $\{\pi^k\}_{k \ge 0}$ generated by BPG converges R-linearly to an element in the critical point set $\mathcal{X}$. 
\end{theorem}

In terms of the convergence analysis of first-order iterative methods solving \eqref{eq:gw_qua}, the \textit{Luo-Tseng Local Error Bound Condition} (cf. Proposition \ref{prop:erb}) has been demonstrated as a crucial tool to unveiling the linear convergence rate under the convex setting~\citep{zhou2017unified,hou2013linear}. The following folkloric result provides a unified template for the convergence rate analysis of first-order methods: 

 \begin{fact}(cf. Fact 1 in \cite{zhou2017unified})
 \label{fact}
 Consider the optimization problem \eqref{eq:gw_qua}, whose the critical point set $\mathcal{X}$ is assumed to be non-empty. Suppose that the generated sequence $\left\{\pi^{k}\right\}_{k \geq 0}$ satisfying the following properties:
 \begin{enumerate}[label=(\Alph*)]
     \item (\textbf{Sufficient Descent}) There exist a constant $\kappa_{1}>0$ and an index $k_{1} \geq 0$ such that for $k \geq k_{1}$,
    \[
    F\left(\pi^{k+1}\right)-F\left(\pi^{k}\right) \leq-\kappa_{1}\left\|\pi^{k+1}-\pi^{k}\right\|^{2}. 
    \]
    \item 
     (\textbf{Cost-to-Go Estimate}) There exist a constant $\kappa_{2}>0$ and an index $k_{2} \geq 0$ such that for $k \geq k_{2}$,
    \[
    F\left(\pi^{k+1}\right)-F(\pi^\star) \leq \kappa_{2}\left(\dist2^{2}\left(\pi^{k}, \mathcal{X}\right)+\left\|\pi^{k+1}-\pi^{k}\right\|^{2}\right),
    \]
    where $\pi^\star$ is the limit point of the sequence $\{\pi^k\}_{k \ge 0}$. 
    \item (\textbf{Safeguard}) There exist a constant $\kappa_{3}>0$ and an index $k_{3} \geq 0$ such that for $k \geq k_{3}$,
    \[
    \left\|R(\pi^{k})\right\|_{2} \leq \kappa_{3}\left\|\pi^{k+1}-\pi^{k}\right\|_{2},
    \]
    where $R(\pi^{k})$ is the proximal residual function, defined as 
    \[
    R(\pi^{k}) = \|\pi^k - \prox_g(\pi^k-\nabla f(\pi^k))\|. 
    \]
 \end{enumerate}
 Suppose further that Problem \eqref{eq:gw_qua} possesses the Luo-Tseng error bound condition. Then, the sequence $\left\{F\left(\pi^{k}\right)\right\}_{k \geq 0}$ converges $Q$-linearly to $F(\pi^\star)$ and the sequence $\left\{\pi^{k}\right\}_{k \geq 0}$ converges $R$-linearly to some $\pi^{\star} \in \mathcal{X}$.
 \end{fact}
 
Before giving the proof details, 
 we would like to highlight the difference from  the technique used in \citep{hou2013linear} and \citep{zhou2017unified}.  The vanilla analysis in \citep{hou2013linear} and \citep{zhou2017unified} can only handle the structured convex problem but GW is nonconvex. We have to exploit the isolation property around the local region to establish the cost-to-go estimate property. Intuitively, the isolation property means that the sequence $\{\pi^k\}_{k \ge 0}$ will eventually settle down at a Euclidean ball. More precisely, we refer the reader to Assumption \ref{ass:iso} for details, which can be verified for the Gromov Wasserstein problem. 
 The geometry from the Bregman divergence can be non-Euclidean. It is not an easy task to extend the analysis in Euclidean space to the general Bregman distance-derived metric space. 
 \begin{assumption}
 \label{ass:iso}
For any  $\bar{\pi} \in \mathcal{x}$, there exists $\delta>0$ so that $F(\pi) = F(\bar{\pi})$ whenever $\pi \in \mathcal{X}$
and $\|\pi-\bar{\pi}\|< \delta $. 
 \end{assumption}

\begin{proof}
\paragraph{\textbf{Step 1:} Sufficient Decrease Property \\}
It is worthwhile noting that $f(\pi)$ is a quadratic function, i.e., $f(\pi) = -\text{Tr}(D_X\pi D_Y\pi^T)$, then $f(\pi)$ is gradient Lipschitz continuous with the modulu constant $\sigma_{\text{max}}(D_X)\sigma_{\text{max}}(D_Y)$, where $\sigma_{\text{max}}(\cdot)$ denotes the largest singular value. To simplify the notation, let $L_f =\sigma_{\text{max}}(D_X)\sigma_{\text{max}}(D_Y)$. 
\begin{align*}
    & F(\pi^{k+1})-F(\pi^k)  \\
    & \leq f(\pi^{k}) + \nabla f(\pi^{k})^T(\pi^{k+1}-\pi^k)+\frac{L_f}{2}\|\pi^{k+1}-\pi^k\|^2 +g (\pi^{k+1}) - f(\pi^k)-g(\pi^k) \\
    &  \stackrel{(\clubsuit)}{\leq} \nabla f(\pi^{k})^T(\pi^{k+1}-\pi^k)+ \frac{L_f}{\sigma}D_h(\pi^{k+1},\pi^k)+g (\pi^{k+1}) -g(\pi^k) \\
    & = \nabla f(\pi^{k})^T(\pi^{k+1}-\pi^k)+ \frac{1}{t_k} D_h(\pi^{k+1},\pi^k)+g (\pi^{k+1}) -g(\pi^k)+\left(\frac{L_f}{\sigma}-\frac{1}{t_k}\right)D_h(\pi^{k+1},\pi^k)\\
    & = \hat{F}(\pi^{k+1};\pi^k) - \hat{F}(\pi^k;\pi^k)+\left(\frac{L_f}{\sigma}-\frac{1}{t_k}\right)D_h(\pi^{k+1},\pi^k)\\
   & \stackrel{(\spadesuit)}{\leq }\left(\frac{L_f}{\sigma}-\frac{1}{t_k}\right)D_h(\pi^{k+1},\pi^k) \\
   & \stackrel{(\clubsuit)}{\leq} -\left(\frac{1}{t_k}-\frac{L_f}{\sigma}\right)\|\pi^{k+1}-\pi^k\|^2, 
\end{align*}
where $(\clubsuit)$ holds due to the fact $D_h(\pi^{k+1}-\pi^k) \ge \frac{\sigma}{2}\|\pi^{k+1}-\pi^k\|^2$ and $(\spadesuit)$ follows as the optimal solution of $\hat{F}(\pi;\pi^{k})$ is $\pi^{k+1}$, i.e., \eqref{eq:bpg_update}. By letting $\kappa_1 =\frac{1}{t_k}-\frac{L_f}{\sigma}>0$, we get the desired result. 

\paragraph{\textbf{Step 2:} Cost-to-Go Estimate Property \\}
Let $\overline{\pi}^k$ be the projection of $\pi^k$ onto $\mathcal{X}$. Again, we aim at exploiting the structure information from the BPG update, i.e., \eqref{eq:bpg_update}. Similarly, as the optimal solution of  $\hat{F}(\pi;\pi^{k})$ is $\pi^{k+1}$, we have, 
\begin{align*}
    & \hat{F}(\pi^{k+1};\pi^{k}) \leq \hat{F}(\overline{\pi}^k;\pi^{k}) \\
    \Rightarrow &  \nabla f(\pi^k)^T \pi^{k+1} + \frac{1}{t_k}D_h(\pi^{k+1},\pi^k) + g(\pi^{k+1}) \leq  \nabla f(\pi^k)^T \overline{\pi}^k + \frac{1}{t_k}D_h(\overline{\pi}^k,\pi^k) + g(\overline{\pi}^k) \\
    \Rightarrow &   \nabla f(\pi^k)^T (\pi^{k+1}-\overline{\pi}^k) + g(\pi^{k+1})-g(\overline{\pi}^k) \leq \frac{1}{t_k}D_h(\overline{\pi}^k,\pi^k)- \frac{1}{t_k}D_h(\pi^{k+1},\pi^k).
\end{align*}
It implies
\begin{equation*}
    \nabla f(\pi^k)^T (\pi^{k+1}-\overline{\pi}^k) + g(\pi^{k+1})-g(\overline{\pi}^k) \leq \frac{1}{t_k}D_h(\overline{\pi}^k,\pi^k)\leq \frac{L}{t_k}\|\overline{\pi}^k-\pi^k\|_F^2 \leq  \frac{L}{\underline{t}}\dist2^2(\pi^k, \mathcal{X}),
\end{equation*}
where $L$ is a constant. 

By the mean value theorem and $f$ is continuous differentiable, there exists a $\hat{\pi}^k \in [\overline{\pi}^k, \pi^{k+1}]$ such that, 
\[
f(\pi^{k+1}) -f(\overline{\pi}^k)= \nabla f(\hat{\pi}^k)^T(\pi^{k+1}-\overline{\pi}^k).
\]
Hence, we compute, 
\begin{align*}
    & F(\pi^{k+1})-F(\overline{\pi}^k) \\
    = &  \nabla f(\hat{\pi}^k)^T(\pi^{k+1}-\overline{\pi}^k)+g(\pi^{k+1})-g(\overline{\pi}^k)\\
    = & \nabla f({\pi}^k)^T(\pi^{k+1}-\overline{\pi}^k) + g(\pi^{k+1})-g(\overline{\pi}^k)+(\nabla f(\hat{\pi}^k)-\nabla f({\pi}^k))^T(\pi^{k+1}-\overline{\pi}^k)\\
    \leq & \frac{L}{\underline{t}}\dist2^2(\pi^k, \mathcal{X}) + (\nabla f(\hat{\pi}^k)-\nabla f({\pi}^k))^T(\pi^{k+1}-\overline{\pi}^k) \\
     \leq & \frac{L}{\underline{t}}\dist2^2(\pi^k, \mathcal{X}) + L_f \|\hat{\pi}^k-{\pi}^k\|\|\pi^{k+1}-\overline{\pi}^k\|\\
    \leq & \frac{L}{\underline{t}}\dist2^2(\pi^k, \mathcal{X}) + L_f \left(\|\hat{\pi}^k-{\pi}^{k+1}\| + \|{\pi}^{k+1}-\pi^k\|\right)\|\pi^{k+1}-\overline{\pi}^k\| \\
    \leq & \frac{L}{\underline{t}}\dist2^2(\pi^k, \mathcal{X}) + L_f\left(\|\overline{\pi}^k-{\pi}^{k+1}\| + \|{\pi}^{k+1}-\pi^k\|\right)\|\pi^{k+1}-\overline{\pi}^k\|\\
    \leq & \frac{L}{\underline{t}}\dist2^2(\pi^k, \mathcal{X}) + \frac{3L_f}{2}\|\overline{\pi}^k-{\pi}^{k+1}\|^2 + \frac{L_f}{2}\|{\pi}^{k+1}-\pi^k\|^2\\
    \leq & \frac{L}{\underline{t}}\dist2^2(\pi^k, \mathcal{X}) + {3L_f}\|\overline{\pi}^k-{\pi}^{k}\|^2 + \frac{7L_f}{2}\|{\pi}^{k+1}-\pi^k\|^2\\
    = & \left(\frac{L}{\underline{t}}+3L_f\right)\dist2^2(\pi^k, \mathcal{X}) + \frac{7L_f}{2}\|{\pi}^{k+1}-\pi^k\|^2 \\
    \leq & \max\left(\left(\frac{L}{\underline{t}}+3L_f\right),\frac{7L_f}{2}\right)\left(\dist2^2(\pi^k, \mathcal{X}) +\|{\pi}^{k+1}-\pi^k\|^2\right).
\end{align*}

For the relative entropy function $h(x) = x\log x$, we can show that $L =\frac{2}{\epsilon} = 2L_h(K)$. Since the sequence $\{\pi^k\}_{k \ge 0}$ is lower bounded by the constant $\epsilon$, i.e., $\pi^k \ge \epsilon I$, then we can conclude 
\[
D_h(\overline{\pi}^k,\pi^k)\leq \frac{2}{\epsilon}\|\overline{\pi}^k-\pi^k\|_F^2,
\]
Let $I =\{(i,j)\mid \overline{\pi}^k_{ij}\geq \frac{\epsilon}{2}\}$ and $\bar{I}$ be its supplementary set. Then, we have
\begin{align*}
    D_h(\overline{\pi}^k,\pi^k) =& \sum_{(i,j)\in I} D_h(\overline{\pi}^k_{ij},\pi^k_{ij}) + \sum_{(i,j)\in \bar{I}} D_h(\overline{\pi}^k_{ij},\pi^k_{ij})\\
    \leq & \sum_{(i,j)\in I} \frac{2}{\epsilon} (\overline{\pi}^k_{ij}-\pi^k_{ij})^2+ \sum_{(i,j)\in \bar{I}} D_h(\overline{\pi}^k_{ij},\pi^k_{ij})\\
    \leq & \sum_{(i,j)\in I} \frac{2}{\epsilon} (\overline{\pi}^k_{ij}-\pi^k_{ij})^2+ \sum_{(i,j)\in \bar{I}} (\pi^k_{ij}-\overline{\pi}^k_{ij})\\
    \leq & \sum_{(i,j)\in I} \frac{2}{\epsilon} (\overline{\pi}^k_{ij}-\pi^k_{ij})^2+ \sum_{(i,j)\in \bar{I}} \frac{2}{\epsilon}(\pi^k_{ij}-\overline{\pi}^k_{ij})^2\\
    =&\frac{2}{\epsilon}\|\overline{\pi} - \pi\|_F^2,
\end{align*}
where the first inequality holds because the second-order direvative $h''(x) = \frac{1}{x} \leq \frac{2}{\epsilon}$ for all $x\geq \frac{\epsilon}{2}$, the second inequality holds because $D_h(x,y) = x\log \frac{x}{y} + y-x \leq y-x$ if $x\leq y$, and the third inequality holds because $\pi^k_{ij}-\overline{\pi}^k_{ij}\geq \epsilon - \frac{\epsilon}{2} = \frac{\epsilon}{2}$ for all $(i,j)\in \bar{I}$.

 At last, by applying Lemma 3.1 in \citep{luo1992error}, we know $F(\cdot)$ enjoys the isolation property around the local region. Therefore, by letting $\kappa_2 =\max\left(\left(\frac{2 L_h(K)}{\underline{t}}+3L_f\right),\frac{7L_f}{2}\right) $, we obtain the desired result. 

\paragraph{\textbf{Step 3:} Safeguard Property \\}
Again, invoking the optimality condition of the main BPG update \eqref{eq:bpg_update}, we have 
\begin{equation}
\label{eq:safeg}
 \begin{aligned}
    & 0 \in \nabla f(\pi^k)+\partial g(\pi^{k+1}) +\frac{1}{t_k}\nabla D_h(\pi^{k+1},\pi^k) \\
    &  0 \in \nabla f(\pi^{k+1})+\partial g(\pi^{k+1})+\nabla f(\pi^k)-\nabla f(\pi^{k+1}) + \frac{1}{t_k}(\nabla h(\pi^{k+1})-\nabla h(\pi^k)).
\end{aligned}
\end{equation}
Upon this, we have 
\begin{align*}
\textnormal{dist}(0,\partial F(\pi^{k+1})) & \leq L_f\|\pi^{k+1}-\pi^k\| + \frac{1}{t_k}\|\nabla h(\pi^{k+1})-\nabla h(\pi^k)\|\\
& \leq \left(L_f+\frac{L_h(K)}{\underline{t}}\right)\|\pi^{k+1}-\pi^k\|. 
\end{align*}
Moreover, due to Lemma 4.1 in \citep{li2018calculus}, it is so nice that we can connect the proximal residual function $R(\pi^k)$ and $\textnormal{dist}(0,\partial F(\pi^{k+1}))$. That is, 
\[
\|R(\pi^{k+1})\| \leq \textnormal{dist}(0,\partial F(\pi^{k+1})).
\]
Then, 
\begin{align*}
\|R(\pi^{k})\| & = \|R(\pi^{k})-R(\pi^{k+1})+R(\pi^{k+1})\|\\
& \leq\|R(\pi^{k})-R(\pi^{k+1})\|+ \|R(\pi^{k+1})\| \\
& \leq \|R(\pi^{k})-R(\pi^{k+1})\| +\textnormal{dist}(0,\partial F(\pi^{k+1})) \\
& = \|\pi^k-\pi^{k+1}+ \prox_g(\pi^{k+1}-\nabla f(\pi^{k+1}))- \prox_g(\pi^k-\nabla f(\pi^k))\|+\textnormal{dist}(0,\partial F(\pi^{k+1}))\\
& \leq \|\pi^k-\pi^{k+1}\| +\|\prox_g(\pi^{k+1}-\nabla f(\pi^{k+1}))- \prox_g(\pi^k-\nabla f(\pi^k))\|+\textnormal{dist}(0,\partial F(\pi^{k+1}))\\
& \leq (L_f+2)\|\pi^k-\pi^{k+1}\|+\textnormal{dist}(0,\partial F(\pi^{k+1}))\\
& \leq \left(2L_f+\frac{L_h(K)}{\underline{t}}+2\right)\|\pi^{k+1}-\pi^k\|. 
\end{align*}

This, together with \eqref{eq:safeg}, implies the safeguard property with $\kappa_3 =\left(2L_f+\frac{L_h(K)}{\underline{t}}+2\right) $.
\end{proof}

\begin{remark}\label{remark1}
Let us comment on the validity of the element-wise lower bound assumption in Theorem \ref{thm:bpg}. Note that  the iterates of BPG will always be strictly positive due to the entropy term. Hence, a finite number of iterations, we can find an $\epsilon$ to satisfy the assumption. Alternatively, we can  add a small perturbation in each iteration to satisfy this assumption, that is, 
$\pi^{k+1} = (1-\epsilon){\pi}^{k+1}  + \epsilon \mathbf{I}_{nm}$. 
Our experience suggests that such a perturbation will almost not affect the performance. It is worth noting that the element-wise lower bound assumption is quite standard in existing analyses of Bregman proximal gradient methods in the literature~\citep{bauschke2017descent,hanzely2021accelerated,beck2017first}. It remains open, even in the optimization literature, whether such an assumption can be removed.
\end{remark}

\section{Experiment Results}\label{sec:exp}
In this section, we provide extensive experiment results to validate the effectiveness of the proposed BAPG and hBPG algorithms on various representative graph learning tasks, including graph alignment, graph partition, and shape matching.
All simulations are implemented using Python 3.9 on a high-performance computing server running Ubuntu 20.04 with an Intel(R) Xeon(R) Gold 6226R CPU and an NVIDIA GeForce RTX 3090 GPU. For all methods conducted in the experiment part, we use  the relative error $\|\pi^{k+1}-\pi^k\|_2 / \|\pi^k\|_2\leq 1e^{-6}$ and the maximum iteration as the stopping criterion, i.e., $\min\{k \in \mathbb{Z}:\|\pi^{k+1}-\pi^k\|_2 / \|\pi^k\|_2\leq 1e^{-6} \, \textnormal{and} \, k\leq 2000 \}$.

\subsection{Toy 2D Matching Problem}
In this subsection, we study a toy matching problem in 2D to corroborate our theoretical insights and results in Sec \ref{sec:method} and \ref{sec:conv}. 
Fig. \ref{fig:syn} (a) shows an example of mapping a two-dimensional shape without any symmetries to a rotated version of the same shape. Here, we sample 300 and 400 points from source and target shapes respectively and use the Euclidean distance to construct the distance matrices $D_X$ and $D_Y$.

\begin{figure*}[!t]
	\centering
	\includegraphics[width=\textwidth]{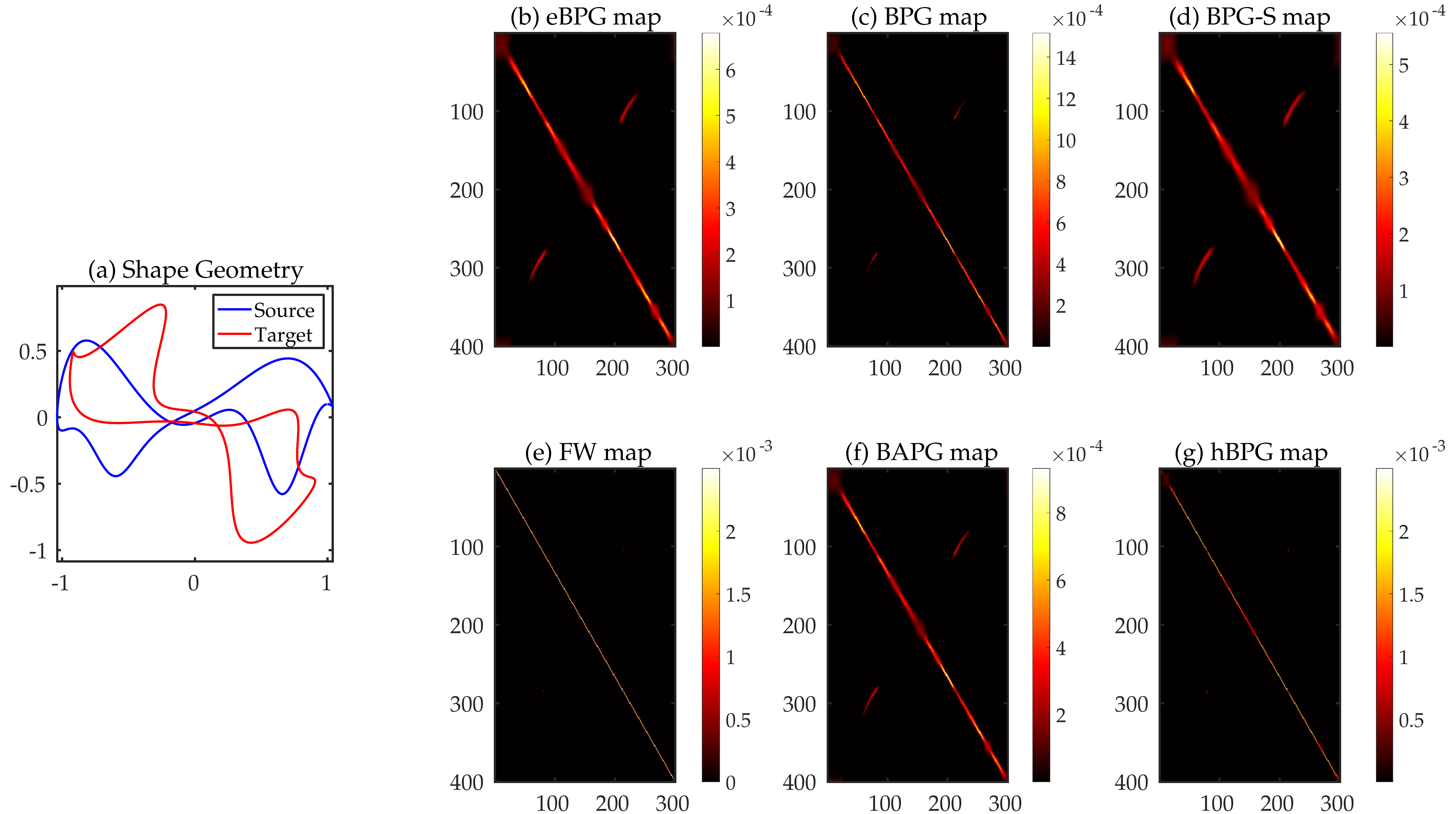}
	\caption{(a): 2D shape geometry of the source and target; (b)-(g): visualization of coupling matrix.} 
\label{fig:syn}
\end{figure*}

Figs.\ref{fig:syn} (b)-(g) provide all color maps of coupling matrices to visualize the matching results. Here, the sparser coupling matrices indicate sharper mapping. All experiment results are consistent with Table \ref{tab:algo}. We can observe that both BPG and FW give us satisfactory solution performance as they aim at solving the GW problem exactly. However, FW and BPG will suffer from a huge computation burden which will be further justified in Sec \ref{sec:graph_align} and \ref{sec:graph_partition}. On the other hand, the performance of eBPG, BPG-S and BAPG are obviously harmed by not solving the original GW problem or the infeasibility issue. The sharpness of BAPG's coupling matrix is relatively not effected by its infeasibility issue too much although its coupling matrix is denser than BPG and FW ones. 
As we shall see later, the effect of the infeasibility issue is minor when the penalty parameter $\rho$ is not too small and will not  even result in even a real cost for graph alignment and partition tasks, which only care about the matching accuracy instead of the sharpness of the coupling.
Moreover, the characteristics of eBPG and BPG motivate us to apply hBPG for shape correspondence problems. Instead, BAPG shows its potential for being applied to graph alignment and graph partition tasks, which only care about the matching accuracy instead of the sharpness of the coupling.



\subsection{Graph Alignment}\label{sec:graph_align} 
Graph alignment aims to identify the node correspondence between two graphs possibly with different topology structures~\citep{zhang2021balancing,chen2020consistent}. Instead of solving the restricted quadratic assignment problem~\citep{lawler1963quadratic,lacoste2006word}, the GW distance provides the optimal probabilistic correspondence relationship via preservation of the isometric property. 
Here, we compare the proposed BAPG and hBPG with all existing baselines: FW~\citep{titouan2019optimal}, BPG~\citep{xu2019gromov}, BPG-S~\citep{xu2019gromov} (i.e., the only difference between BPG and BPG-S is that the number of inner iterations  for BPG-S is just one), ScalaGW~\citep{xu2019scalable}, SpecGW~\citep{chowdhury2021generalized}, and eBPG~\citep{solomon2016entropic}. Except for BPG and eBPG, others are pure heuristic methods without any theoretical guarantee. Besides the GW-based methods, we also consider three widely used non-GW graph alignment baselines, including IPFP \citep{ipfp}, RRWM \citep{rrwm}, and SpecMethod \citep{sm}.

\paragraph{Parameters Setup} We utilize the unweighted symmetric adjacent matrices as our input distance matrices, i.e., $D_X$ and  $D_Y$. Alternatively,  SpecGW uses the heat kernel $\exp(-L)$ where $L$ is the normalized graph Laplacian matrix. We set both $\mu$ and $\nu$ to be the uniform distribution.  For three heuristic methods --- BPG-S, ScalaGW, and SpecGW, we follow the same setup reported in their papers. As mentioned, eBPG is very sensitive to the entropic regularization parameter. To get comparable results, we report the best result among the set $\{0.1, 0.01, 0.001\}$ of the regularization parameter. For BPG and BAPG, we use the constant step size $5$ and $\rho=0.1$ respectively. For FW, we use the default implementation in the PythonOT package \citep{flamary2021pot}. 

\begin{table}[h]
\centering
\caption{Statistics of databases for graph alignment.}
\begin{tabular}{c|ccc}
\toprule
Dataset    & \# Samples  & Ave. Nodes  & Ave. Edges \\\midrule
Synthetic   & 300     & 1500        & 56579 \\
Proteins    & 1113    & 39.06       & 72.82 \\
Enzymes     & 600     & 32.63       & 62.14 \\
Reddit      & 500     & 375.9       & 449.3 \\\bottomrule
\end{tabular}

\label{tab:align_data}
\end{table}

 \begin{table*}[!t]
\caption{Comparison of the matching accuracy (\%) and wall-clock time (seconds) on graph alignment. For BAPG, we also report the time of GPU implementation.}
\centering
\small
\resizebox{\linewidth}{!}{
\begin{tabular}{c|cr|ccr|ccr|ccr}
\toprule

\multirow{2}{*}{Method} &\multicolumn{2}{c|}{Synthetic} & \multicolumn{3}{c|}{Proteins}  & \multicolumn{3}{c|}{Enzymes} & \multicolumn{3}{c}{Reddit} \\
            & Acc & Time & Raw & Noisy & Time & Raw & Noisy & Time & Raw & Noisy & Time  \\\midrule
IPFP         & - & - & 43.84 & 29.89 & 87.0 & 40.37 & 27.39 & 23.7 & - & - & -\\
RRWM         & - & - & 71.79 & 33.92 & 239.3& 60.56 & 30.51 & 114.1 & - & - & -\\
SpecMethod   & - & - & 72.40 & 22.92 & 40.5 & 71.43 & 21.39 & 9.6 & - & - & -\\\midrule
FW & 24.50 & 8182 & 29.96 & 20.24 & 54.2 & 32.17 & 22.80 & 10.8 & 21.51 & 17.17 & 1121\\
ScalaGW & 17.93& 12002& 16.37 & 16.05 &372.2  & 12.72 & 11.46 &213.0&  0.54 & 0.70 &1109\\
SpecGW & 13.27&  1462& 78.11 & 19.31 &\textbf{30.7}  & 79.07 & 21.14 & \textbf{6.7}& 50.71 &19.66 &1074\\
eBPG & 34.33&  9502& 67.48 & 45.85 &208.2 & 78.25 & 60.46 &499.7&  3.76 & 3.34 &1234\\
BPG         & 57.56& 22600 & 71.99 & 52.46 &130.4 & 79.19 & 62.32 &73.1& 39.04 &36.68 &1907\\
BPG-S & 61.48& 18587& 71.74 & 52.74 &40.4  & 79.25 & 62.21 &13.4& 39.04 &36.68 &1431\\\midrule
hBPG        & 51.57 & 13279 & 70.07 & 49.01 & 245.9 & 78.57 & 62.26 & 560.0 & 47.15 & 45.58 & 1447\\
BAPG        & \textbf{99.79}&  9024& \textbf{78.18} & \textbf{57.16} &59.1  & \textbf{79.66} & \textbf{62.85} &14.8& \textbf{50.93} &\textbf{49.45} & 780 \\
BAPG-GPU    & - & \textbf{1253} & - & - & 75.4 & - & - & 21.8 & - & - &\textbf{115}\\\bottomrule
\end{tabular}
}
\label{tab:align_result} 
\end{table*}

\paragraph{Database Statistics} We test all methods on both synthetic and real-world databases. Our setup for the synthetic database is the same as in \citep{xu2019gromov}. The source graph $\mathcal G_s = \{\mathcal V_s, \mathcal E_s\}$ is generated by two ideal random models, Gaussian random partition and Barabasi-Albert models, with different scales, i.e., $|\mathcal V_s| \in \{500, 1000, 1500, 2000, 2500\}$.
Then, we generate the target graph $\mathcal G_t=\{\mathcal V_t, \mathcal E_t\}$ by first adding $q\%$ noisy nodes to the source graph, and then generating $q\%$ noisy edges between the nodes in $\mathcal V_t$, i.e., $|\mathcal V_t|=(1+q\%)|\mathcal V_s|, |\mathcal E_t|=(1+q\%)|\mathcal E_s|$, where $q\in \{0, 10, 20, 30, 40, 50\}$. For each setup, we generate five synthetic graph pairs over different random seeds. To sum up, the synthetic database contains 300 different graph pairs. We also validate our proposed methods on other three real-world databases from \citep{chowdhury2021generalized}, including two biological graph databases \textit{Proteins} and \textit{Enzymes}, and a social network database \textit{Reddit}. Furthermore, to demonstrate the robustness of our method regarding the noise level, we follow the noise generating process (i.e., $q = 10\%$) conducted for the synthesis case to create new databases on top of the three real-world databases. Towards that end, the statistics of all databases used for the graph alignment task have been summarized in Table \ref{tab:align_data}. We match each node in $\mathcal G_s$ with the most likely node in $\mathcal G_t$ according to the optimized $\pi^\star$. Given the predicted correspondence set $\mathcal S_{\text{pred}}$ and the ground-truth correspondence set $\mathcal S_{\text{gt}}$, we calculate the matching accuracy by $\textnormal{Acc}=|\mathcal S_{\text{gt}} \cap \mathcal S_{\text{pred}}|/|\mathcal S_{\text{gt}}|\times 100\%$.

\paragraph{Results of Our Methods} Table \ref{tab:align_result} shows the
comparison of matching accuracy and wall-clock time on four databases. We observe that BAPG works exceptionally well both in terms of computational time and accuracy, 
especially for two large-scale noisy graph databases \textit{Synthetic} and \textit{Reddit}. Notably, BAPG is robust enough so that it is not necessary to perform parameter tuning.
As we mentioned in Sec \ref{sec:method}, the effectiveness of GPU acceleration for BAPG is also well corroborated on \textit{Synthetic} and \textit{Reddit}. GPU cannot further speed up the training time of \textit{Proteins} and \textit{Reddit} as graphs in these two databases are small-scale. 
Additional experiment results to demonstrate the robustness of BAPG and its GPU acceleration will be given in Sec \ref{sec:sentiv}.
Observed from Table \ref{tab:align_result}, the performance of hBPG is in general the interpolation between eBPG and BPG in terms of modeling performance and wall-clock time. 


\paragraph{Comparison with Other Methods} Traditional non-GW graph alignment methods (IPFP, RRWM, and SpecMethod) have the out-of-memory issue on graphs with more than 500 nodes (e.g., Synthetic and Reddit) and are sensitive to the noise. The performance of eBPG and ScalaGW are influenced by the entropic regularization parameter and approximation error respectively, which accounts for their poor performance. Moreover, it is easy to observe that SpecGW works pretty well on the small dataset but the performance degrades dramatically on the large one, e.g., \textit{synthetic}.  The reason is that SpecGW relies on a linear programming solver as its subroutine, which is not well-suited for large-scale settings. Besides, although ScalaGW has the lowest per-iteration computational complexity, the recursive K-partition mechanism developed in \citep{xu2019scalable} is not friendly to parallel computing. Therefore, ScalaGW does not demonstrate attractive performance on multi-core processors.

\subsection{Graph Partition}\label{sec:graph_partition}
The GW distance can also be potentially applied to the graph partition task. That is, we are trying to match the source graph with a disconnected target graph having $K$ isolated and self-connected super nodes, where $K$ is the number of clusters~\citep{abrishami2020geometry}. Similarly, we compare the proposed BAPG and hBPG with other baselines described in Section \ref{sec:graph_align} on four real-world graph partitioning datasets. Following \citep{chowdhury2021generalized}, we also add three non-GW methods specialized in graph alignment, including FastGreedy \citep{clauset2004finding}, Louvain \citep{blondel2008fast}, and Infomap \citep{rosvall2008maps}.

\paragraph{Parameters Setup} For the input distance matrices $D_X$ and $D_Y$, we test our methods on both the adjacent matrices and the heat kernel matrices proposed in \citep{chowdhury2021generalized}. For BAPG, we pick up the lowest converged function value among $\rho \in\{0.1,0.05,0.01\}$ for adjacent matrices and $\rho \in\{0.01,0.005,0.001\}$ for heat kernel matrices. The quality of graph partition results is quantified by computing the adjusted mutual information (AMI) score \citep{vinh2010information} against the ground-truth partition. 
\begin{table*}[]
\centering
\caption{Comparison of AMI scores on graph partition datasets using the adjacent matrices and the heat kernel matrices.}
\resizebox{\linewidth}{!}{
\begin{tabular}{cc|cccccccc}
\toprule
\multirow{2}{*}{Matrices} & \multirow{2}{*}{Method} & \multicolumn{2}{c}{Wikipedia}   & \multicolumn{2}{c}{EU-email}    & \multicolumn{2}{c}{Amazon}      & \multicolumn{2}{c}{Village}     \\
                           &                          & Raw            & Noisy          & Raw            & Noisy          & Raw            & Noisy          & Raw            & Noisy          \\ \midrule
\multirow{3}{*}{Non-GW}& 
FastGreedy  &0.382&0.341&0.312&0.251&0.637&0.573&\textbf{0.881}&0.778\\
& Louvain   &0.377&0.329&0.447&0.382&0.622&\textbf{0.584}&\textbf{0.881}&\textbf{0.827}\\
& Infomap   &0.332&0.329&0.374&0.379&\textbf{0.940}&0.463&\textbf{0.881}&0.190\\\midrule
\multirow{5}{*}{Adjacent}  & FW &0.341 & 0.323 & 0.440 & 0.409 & 0.374 & 0.338 & 0.684 & 0.539\\
                           & eBPG                     & 0.461          & \textbf{0.413} & \textbf{0.517} & 0.422 & 0.429          & 0.387          & 0.703          & 0.658          \\
                           & BPG                      & 0.367          & 0.333          & 0.478          & 0.414          & 0.412          & 0.368          & 0.642          & 0.575          \\
                           & BPG-S                      & 0.357          & 0.285          & 0.451          & 0.404          & \textbf{0.443} & 0.352          & 0.606          & 0.560          \\

                            & hBPG                    & 0.368                           & 0.333                           & 0.527                           & 0.423                           & 0.435                           & 0.387                           & 0.655                           & 0.554                           \\

                           & BAPG                     & \textbf{0.468} & 0.385          & 0.508          & \textbf{0.428}          & 0.436          & \textbf{0.426} & \textbf{0.709} & \textbf{0.681} \\ \midrule
\multirow{5}{*}{Heat Kernel}  & SpecGW                  & 0.442          & 0.395          & 0.487          & 0.425          & 0.565          & 0.487          & 0.758          & 0.707          \\
                           & eBPG                     & 0.000          & 0.000          & 0.000          & 0.000          & 0.000          & 0.000          & 0.000          & 0.000          \\
                           & BPG                      & 0.405          & 0.373          & 0.473          & 0.253          & 0.492          & 0.436          & 0.705          & 0.619          \\
                            & BPG-S                      & 0.411          & 0.373          & 0.475          & 0.253          & 0.483          & 0.425          & 0.642          & 0.619          \\

                            & hBPG                    & 0.497                           & 0.387                           & 0.166                           & 0.059                           & 0.477                           & 0.389                           & 0.782                           & 0.727\\

                           & BAPG                     & \textbf{0.529} & \textbf{0.397} & \textbf{0.533} & \textbf{0.436} & \textbf{0.609} & \textbf{0.505} & \textbf{0.797} & \textbf{0.711} \\ \bottomrule
\end{tabular}
}
\label{tab:partition_result} 
\end{table*}

\paragraph{Results of All Methods} Table \ref{tab:partition_result} shows the comparison of AMI scores among all methods on graph partition. BAPG outperforms other GW-based methods in most cases and is more robust to the noisy setting. Specifically, BAPG is consistently better than FW and SpecGW which both rely on the Frank-Wolfe method to solve the problem. eBPG has comparable results using the adjacency matrices but is unable to process the spectral matrices. The possible reason is that the adjacency matrix and the heat kernel matrix admit quite different structures, e.g., the former is sparse while the latter is dense. The performance of hBPG is in general the
interpolation between eBPG and BPG. BPG and BPG-S enjoy similar performances in most cases, but both are not as good as our proposed BAPG on all datasets. Except for BPG and eBPG, others are pure heuristic methods without any theoretical guarantee. BAPG also shows competitive performance compared to the specialized non-GW graph partition methods. For example, BAPG outperforms Infomap in 6 out of 8 scores.

\subsection{Shape Correspondence}\label{sec:shape}

\begin{figure*}[h!]
	\centering
	\vspace{-3mm}
	\subfigure[Source Surface]{\includegraphics[width=0.24\columnwidth]{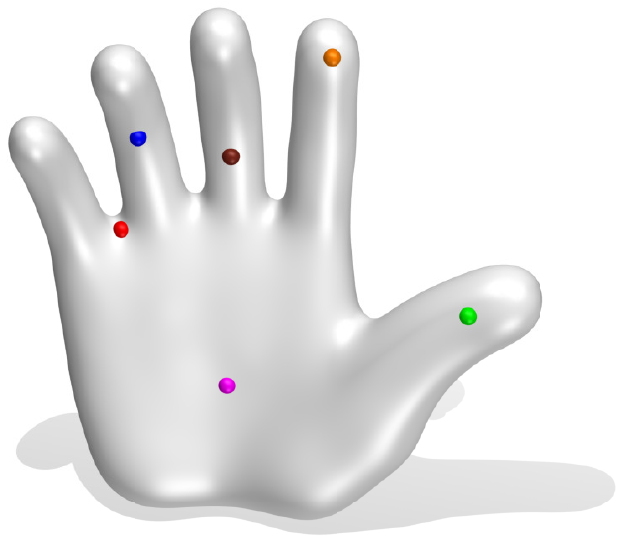}}
	\subfigure[eBPG]{\includegraphics[width=0.16\textwidth]{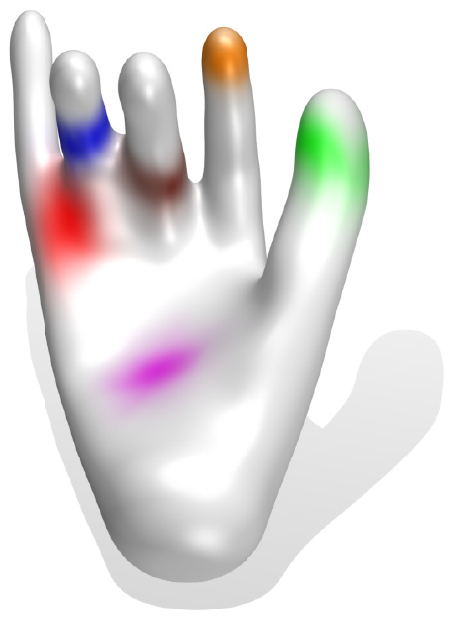}} 
	\subfigure[BAPG]{\includegraphics[width=0.16\textwidth]{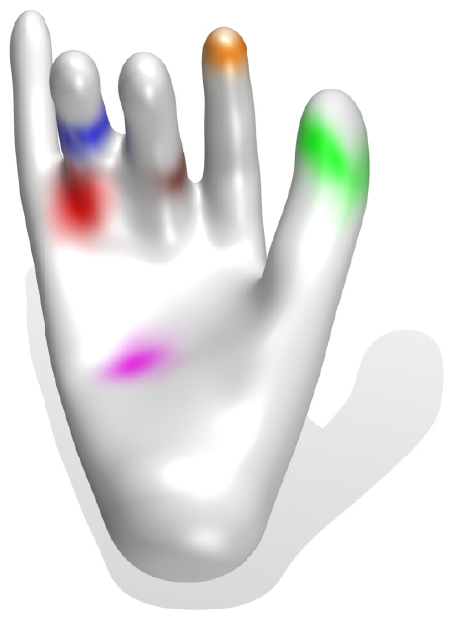}} 
	\subfigure[BPG]{\includegraphics[width=0.16\textwidth]{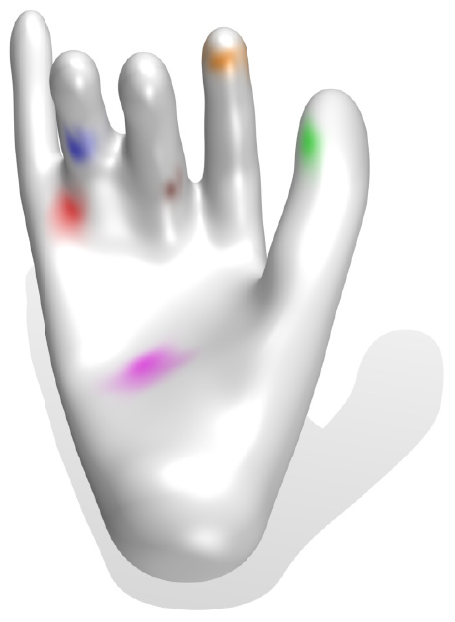}} 
			\vspace{-3mm}
	\subfigure[hBPG]{\includegraphics[width=0.16\textwidth]{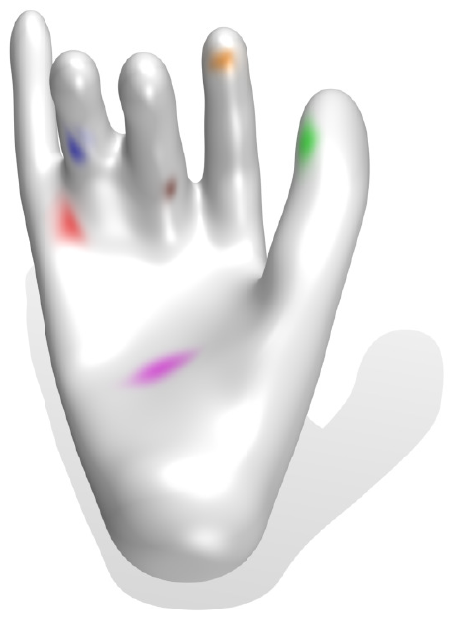}} 
	\subfigure[Source Surface]{\includegraphics[width=0.22\textwidth]{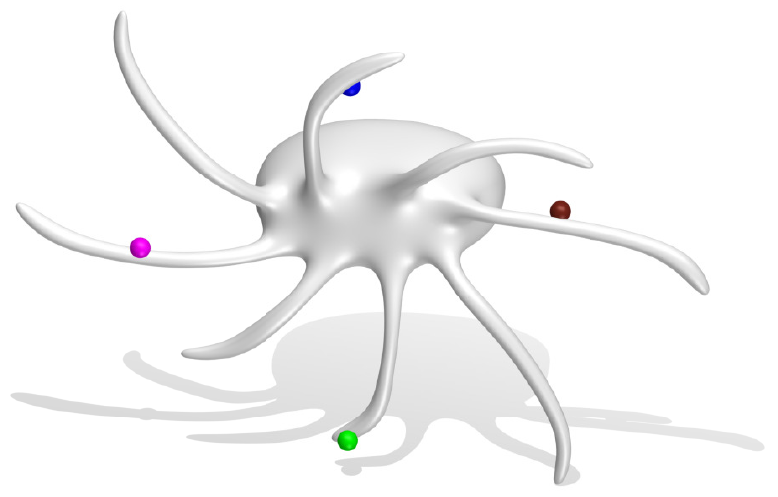}} 
	\subfigure[eBPG]{\includegraphics[width=0.18\textwidth]{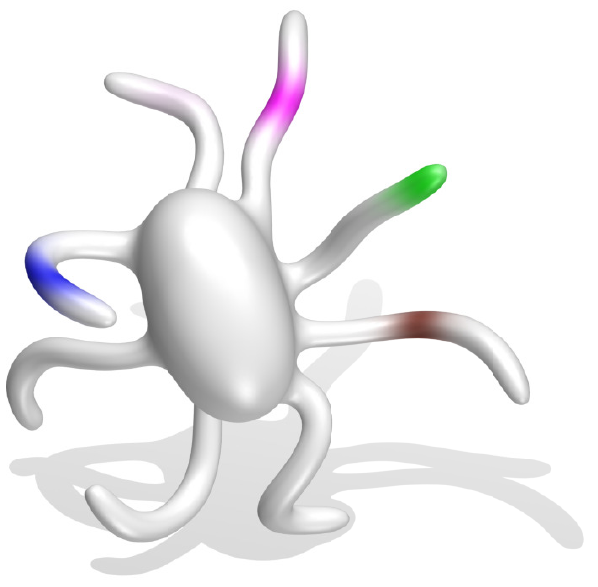}} 
	\subfigure[BAPG]{\includegraphics[width=0.18\textwidth]{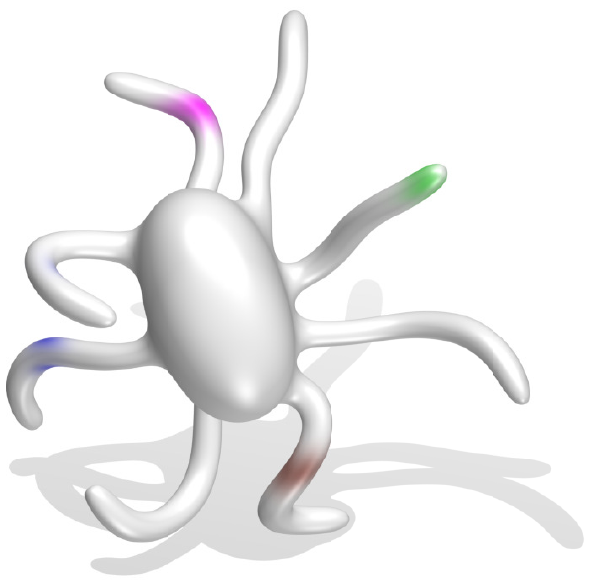}} 
	\subfigure[BPG]{\includegraphics[width=0.18\textwidth]{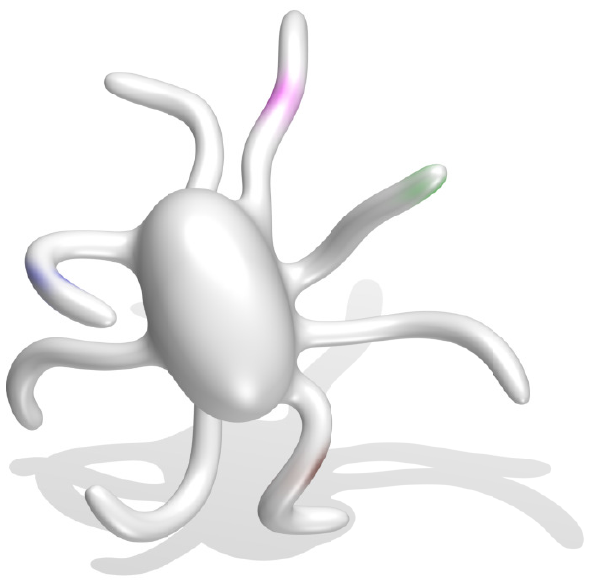}} 
	\subfigure[hBPG]{\includegraphics[width=0.18\textwidth]{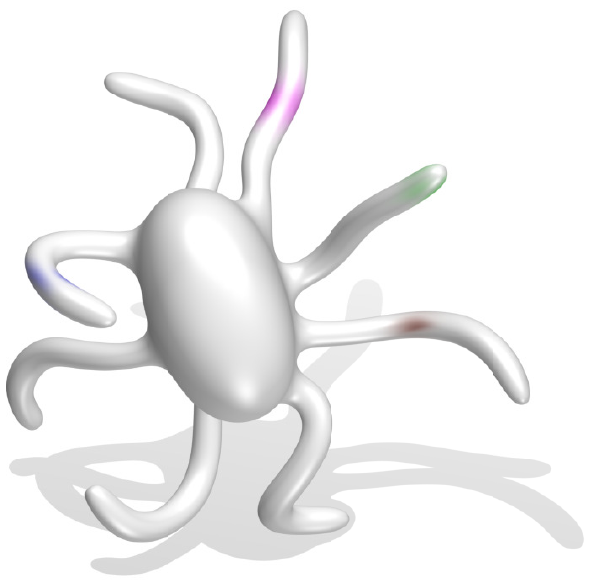}} 
		\vspace{-3mm}
	\subfigure[Source Surface]{\includegraphics[width=0.2\textwidth]{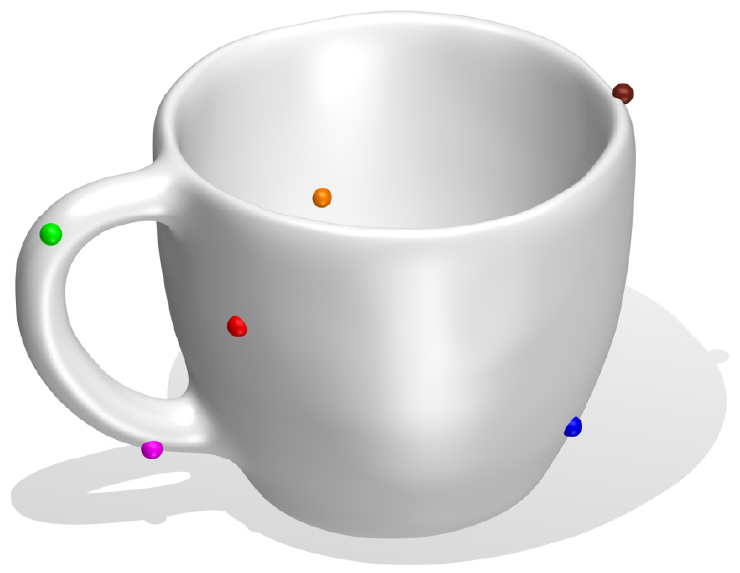}}
	\subfigure[eBPG]{\includegraphics[width=0.18\textwidth]{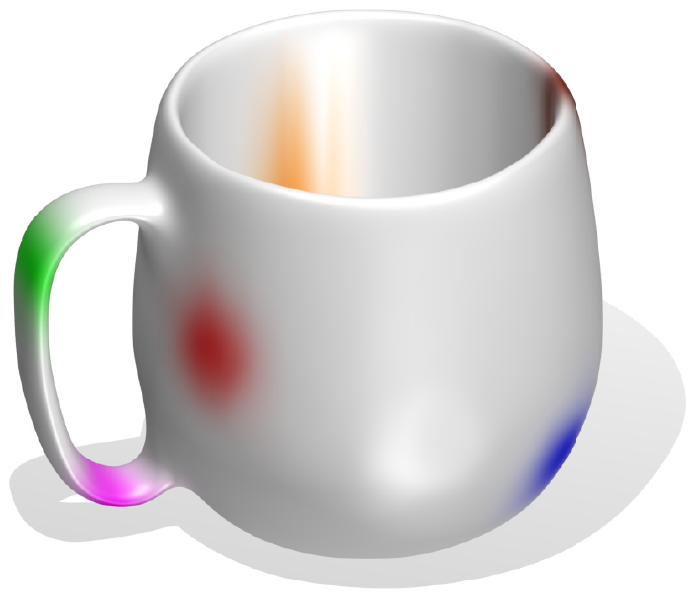}} 
	\subfigure[BAPG]{\includegraphics[width=0.18\textwidth]{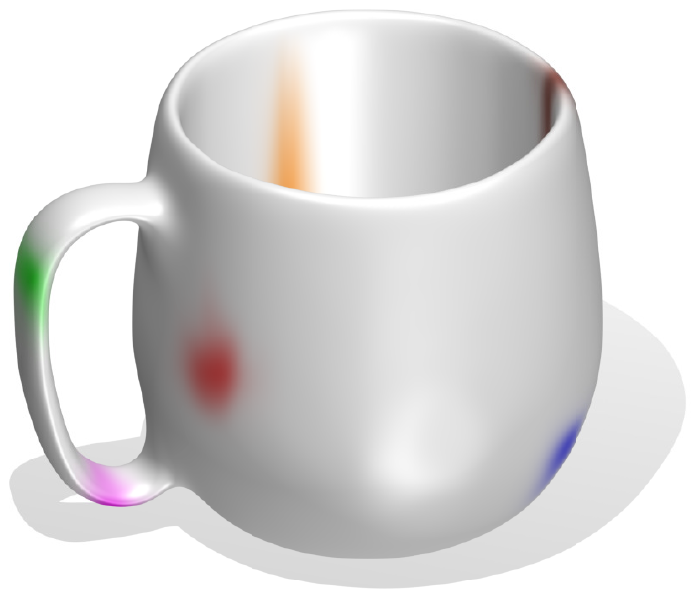}} 
	\subfigure[BPG]{\includegraphics[width=0.18\textwidth]{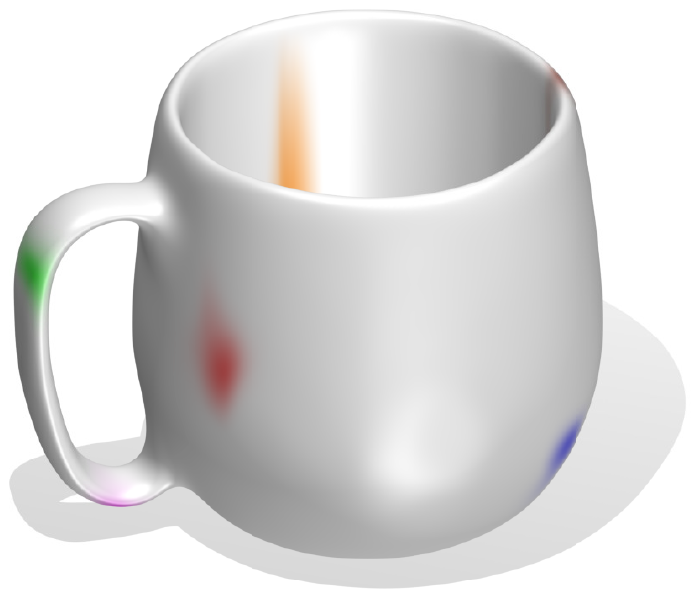}} 
	\subfigure[hBPG]{\includegraphics[width=0.18\textwidth]{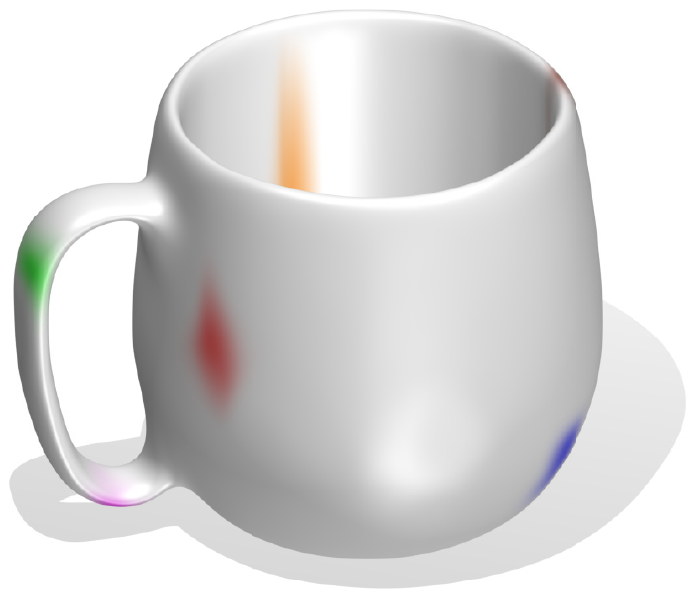}} 
	\subfigure[Source Surface]{\includegraphics[width=0.18\textwidth, height=0.2\textwidth]{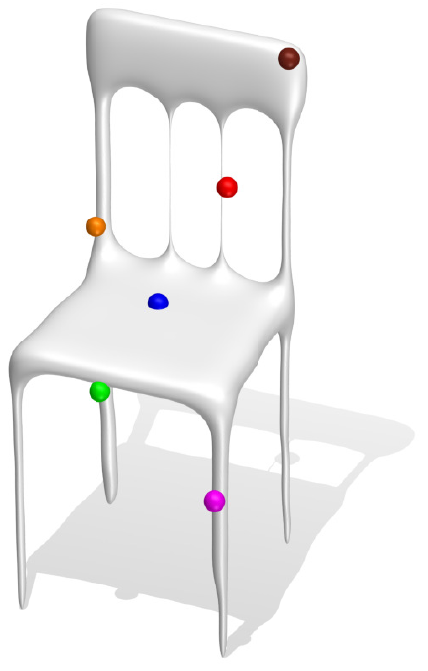}}
	\subfigure[eBPG]{\includegraphics[width=0.16\textwidth,height=0.22\textwidth]{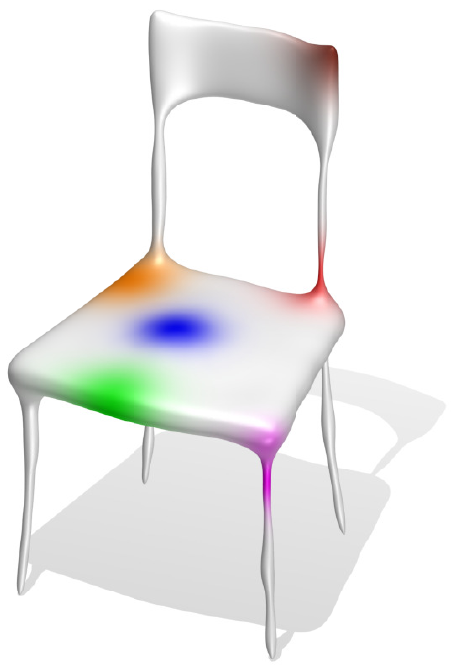}} 
	\subfigure[BAPG]{\includegraphics[width=0.16\textwidth,height=0.22\textwidth]{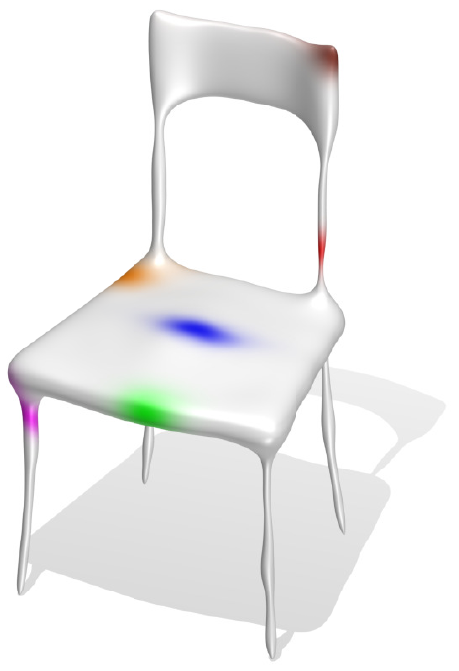}} 
	\subfigure[BPG]{\includegraphics[width=0.16\textwidth,height=0.22\textwidth]{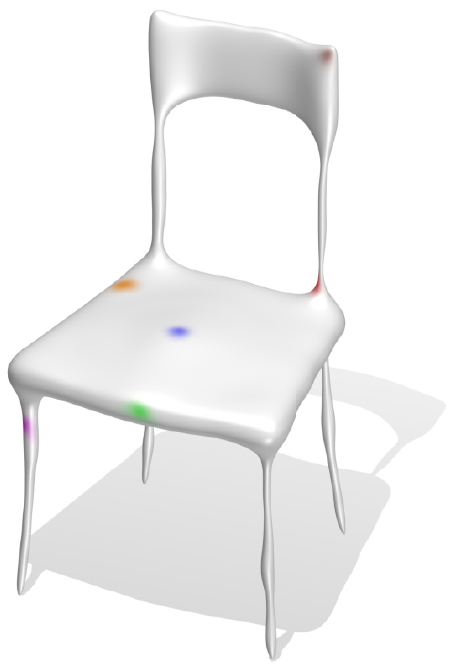}} 
	\subfigure[hBPG]{\includegraphics[width=0.16\textwidth,height=0.22\textwidth]{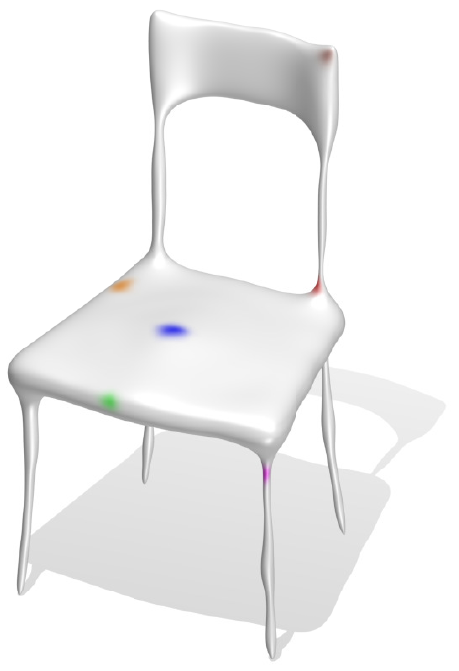}} 
	\subfigure[Source Surface]{\includegraphics[width=0.19\textwidth,height=0.22\textwidth]{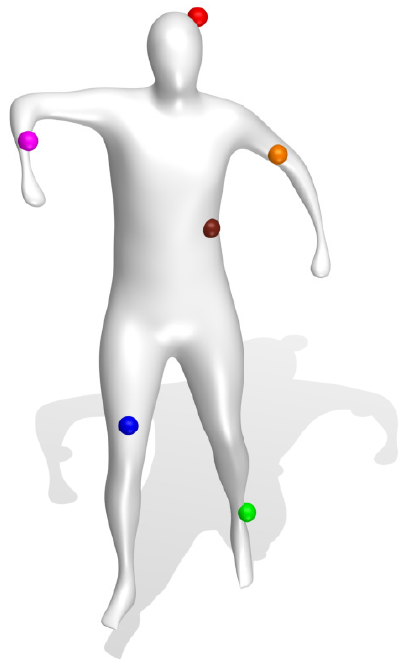}} 
	\subfigure[eBPG]{\includegraphics[width=0.14\textwidth,height=0.24\textwidth]{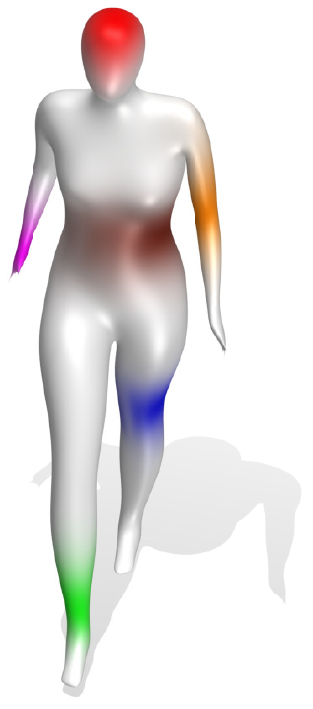}} 
	\subfigure[BAPG]{\includegraphics[width=0.14\textwidth,height=0.24\textwidth]{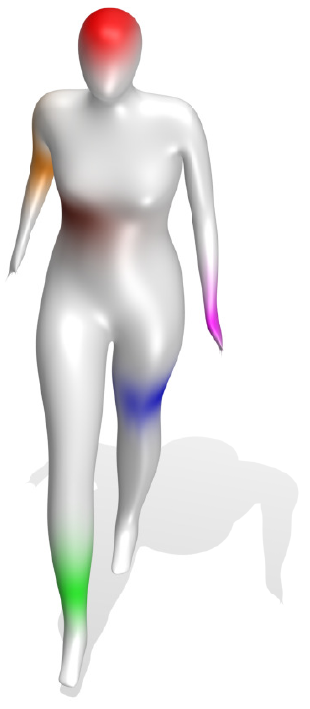}} 
	\subfigure[BPG]{\includegraphics[width=0.14\textwidth,height=0.24\textwidth]{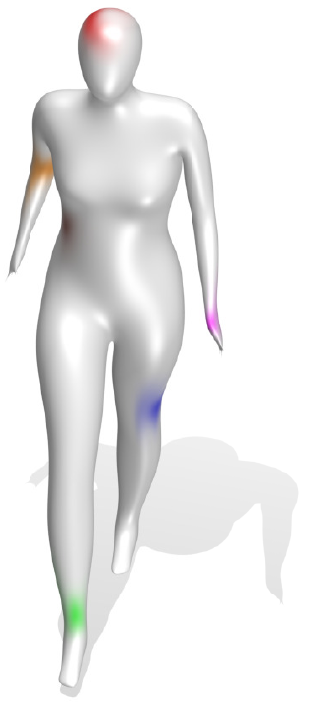}} 
	\subfigure[hBPG]{\includegraphics[width=0.14\textwidth,height=0.24\textwidth]{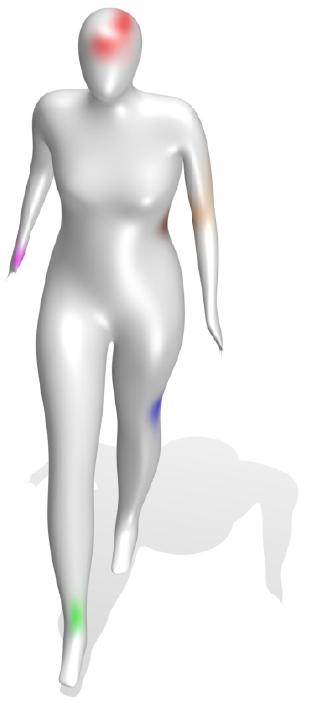}}
	\caption{Visualization of the matching results: colored points (source surface) are mapped to colored distributions on the target object.} \label{fig:hand}
\end{figure*}

\begin{table*}[t]
\caption{Comparison of the infeasibility error (i.e.,$\tfrac{\|\pi^T\mathbf{1}_n -\nu\|}{m}+\tfrac{\|\pi \mathbf{1}_m -\mu\|}{n}$ ) and CPU wall-clock time.}
\resizebox{\textwidth}{!}{
\begin{tabular}{@{}c|cc|cc|cc|cc|cc@{}}
\toprule
\multirow{2}{*}{Method} & \multicolumn{2}{c|}{Hand} & \multicolumn{2}{c|}{Octopus} & \multicolumn{2}{c|}{Mug} & \multicolumn{2}{c|}{Chair} & \multicolumn{2}{c}{Human} \\
                        & Error        & Time       & Error          & Time        & Error        & Time      & Error         & Time       & Error         & Time      \\ \midrule
eBPG                    & 2.95e-10     & 9.37       & 7.21e-10       & 12.08       & 6.41e-10     & 27.86     & 1.77e-10      & 12.87      & 5.52e-10      & 11.46     \\
BPG                     & 3.00e-07     & 389.72     & 2.00e-07       & 32.55       & 3.50e-07     & 196.93    & 4.00e-07      & 304.98     & 2.53e-07      & 93.27     \\
hBPG                    & 3.00e-07     & 193.85     & 2.00e-07       & 23.14        & 3.50e-07     & 90.59     & 4.00e-07      & 189.41     & 2.53e-07      & 53.29     \\
BAPG                    & 4.54e-06     & 61.77      & 2.14e-05       & 6.19        & 6.62e-05     & 30.83     & 2.13e-05      & 127.78     & 5.62e-05      & 8.26      \\
BAPG-GPU                & -            & 3.10       & -              & 1.28        & -            & 1.39      & -             & 3.22       & -             & 0.78      \\ \bottomrule
\end{tabular}
}
\vspace{-1.2\baselineskip}
\label{tab:geometry_result} 
\end{table*}

Finally, we evaluate the matching performance and computational cost of our proposed  hBPG and BAPG on five 3D triangle mesh datasets used in \cite{solomon2016entropic}.
By making use of gptoolbox~\citep{jacobson2018gptoolbox}, $D_X$ and $D_Y$ are constructed by computing the geodesic distances over the triangle mesh.  Here, $\mu$ and $\nu$ are discrete uniform distributions. All the algorithmic parameters of eBPG, BPG, BAPG and hBPG have been fine-tuned via grid search for optimal performance.

Different from graph alignment and partition, the shape matching task cares more about the sharpness of the correspondence relationship. To visualize this soft matching result, we label several color points on the source surface and then quantify the sharpness of each mapping distribution via the size of the colored area on the target surface. In fact, the smaller area indicates a sharper mapping. Regarding the accuracy and sharpness, the superiority of hBPG and BPG over other methods are obviously observed from Fig. \ref{fig:hand}. If we further take the computational burden into account, hBPG is the best-fit for this task, see Table \ref{tab:geometry_result} for details. hBPG can achieve a great trade-off between efficiency and accuracy. Although BAPG shows attractive advantages in terms of computational cost, it suffers from the infeasibility issue, which is also corroborated by the experiment results in Table \ref{tab:geometry_result}. The visualization results of the other four 3D mesh objects are given in Appendix.

\subsection{The Effectiveness and Robustness of BAPG}
\label{sec:sentiv}
At first, we target at demonstrating the robustness of BAPG on graph alignment, as it would be more reasonable to test the robustness of a method on a database (e.g., graph alignment) rather than a single point (e.g., graph partition).

\paragraph{Noise Level and Graph Scale} At the beginning, we report the sensitivity analysis of BAPG regarding the noise level $q\%$ and the graph scale $|\mathcal V_s|$ in Fig. \ref{fig:noise} on the Synthetic Database. Surprisingly, the solution performance of BAPG is robust to both the noise level and graph scale. In contrast, the accuracy of other methods degrades dramatically as the noise level or the graph scale increases. 

\paragraph{Step Size $\rho$}
Subsequently, we conduct the sensitivity analysis of BAPG with respect to the step size $\rho$ on four real-world databases. Table \ref{tab:sensit} showcases that the matching performance of BAPG is stable and robust (not sensitive) regarding the step size.
Moreover, convergence curves of different  $\rho$ in Fig. \ref{fig:noise} corroborate Proposition \ref{prop:bapg_approx} empirically and further guide us on how to choose the step size. The larger $\rho$ means the smaller infeasibility error but slower convergence speed. Together with Table \ref{tab:sensit}, we can conclude that if the resulting infeasibility error is not quite large (i.e., the step size $\rho$ is not so small),  the final matching accuracy is robust to the step size. Under this situation, the step size only affects the convergence speed. 

\begin{table}[]
\caption{The sensitivity analysis of BAPG with different step sizes $\rho$ on graph alignment.}
\resizebox{\linewidth}{!}{
\begin{tabular}{c|cr|ccr|ccr|ccr}
\toprule
\multirow{2}{*}{Method} &\multicolumn{2}{c|}{Synthetic} & \multicolumn{3}{c|}{Proteins}  & \multicolumn{3}{c|}{Enzymes} & \multicolumn{3}{c}{Reddit} \\
                               & Acc       & Time & Clean   & Noise & Time & Clean  & Noise & Time & Clean  & Noise & Time \\\midrule
BAPG ($\rho$=0.5)  & 94.58     & 4610 & 77.21   & 57.69 & 70.8 & 79.20  & 63.18 & 21.8 & 50.04  & 49.19 & 257  \\
BAPG ($\rho$=0.2)  & 98.87     & 2074 & 78.27   & 57.96 & 67.4 & 79.50  & 63.06 & 17.6 & 50.94  & 50.46 & 144  \\
BAPG ($\rho$=0.1)  & 99.79     & 1253 & 78.18   & 57.16 & 59.1 & 79.66  & 62.85 & 14.8 & 50.93  & 49.45 & 115  \\
BAPG ($\rho$=0.05) & 99.20     & 675  & 77.59   & 56.54 & 32.6 & 79.27  & 62.13 & 11.8 & 50.96  & 50.30 & 98   \\
BAPG ($\rho$=0.01) & 93.63     & 616  & 60.08   & 34.47 & 22.5 & 66.45  & 42.38 & 8.9  & 28.74  & 16.14 & 85 \\ 
\bottomrule
\end{tabular}
}
\label{tab:sensit}
\end{table}

\begin{table}[t!]
\centering
\label{tab:gpu}
\caption{GPU \& CPU wall-clock time comparison of BAPG, BPG and eBPG on graph alignment.}
\resizebox{0.7\linewidth}{!}{

\begin{tabular}{c|ccc|ccc}
\toprule
                    & \multicolumn{3}{c|}{Reddit Dataset} & \multicolumn{3}{c}{Synthetic Dataset} \\
                    & BAPG             & BPG    & eBPG   & BAPG             & BPG      & eBPG    \\\midrule
CPU Time(s) & 780              & 1907   & 1234   & 9024             & 22600    & 9502    \\
GPU  Time(s) & 115              & 1013   & 2274   & 1253             & 4458     & 2709    \\
Acceleration Ratio  & \textbf{6.78}    & 1.88   & 0.54   & \textbf{7.20}    & 5.07     & 3.51   \\\bottomrule
\end{tabular}
}
\label{tab:gpu}
\end{table}

\begin{figure*}[t!]
	\centering
	\subfigure[]{\includegraphics[width=0.66\textwidth]{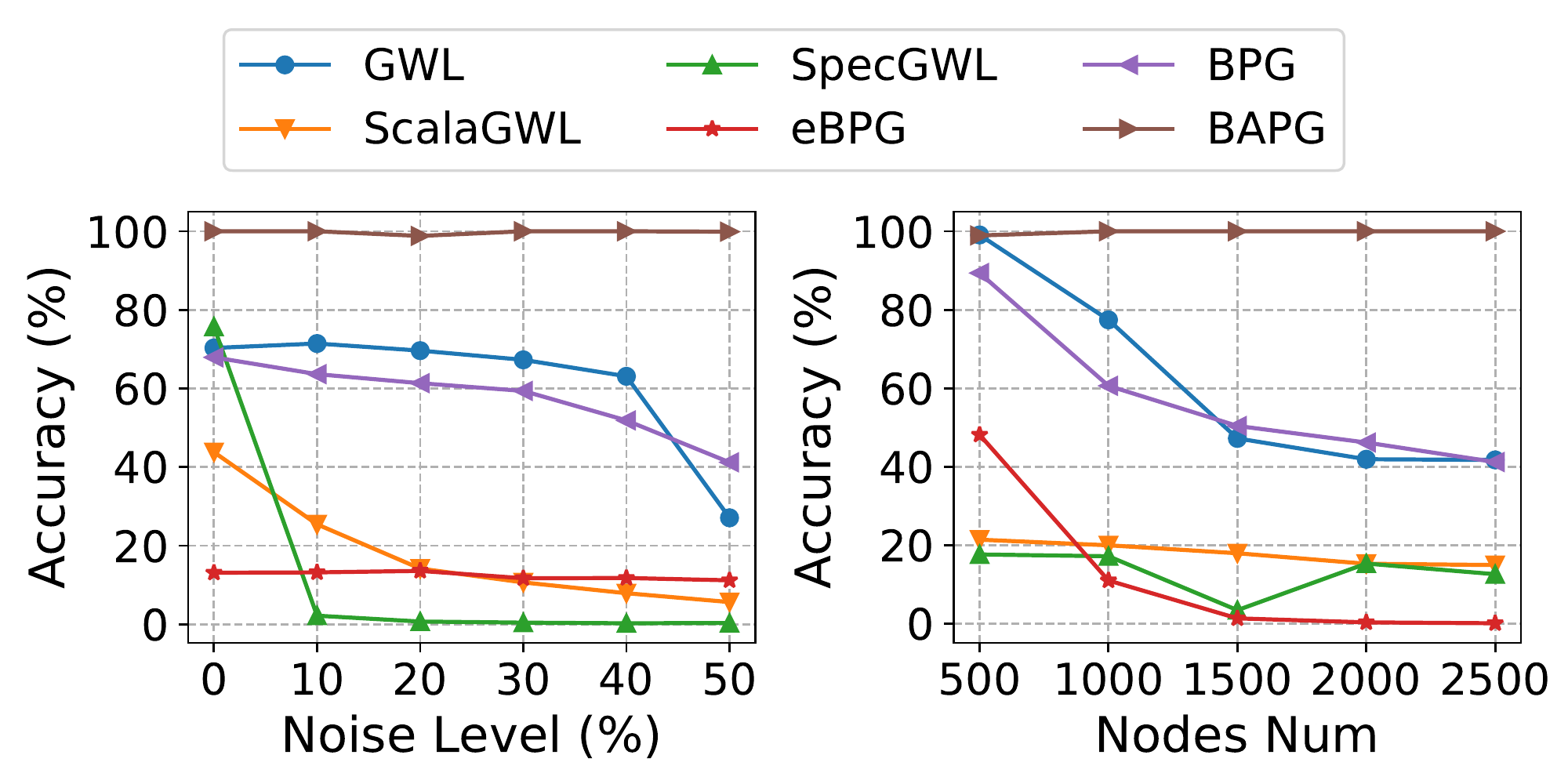}} 
	\subfigure[]{\includegraphics[width=0.33\textwidth]{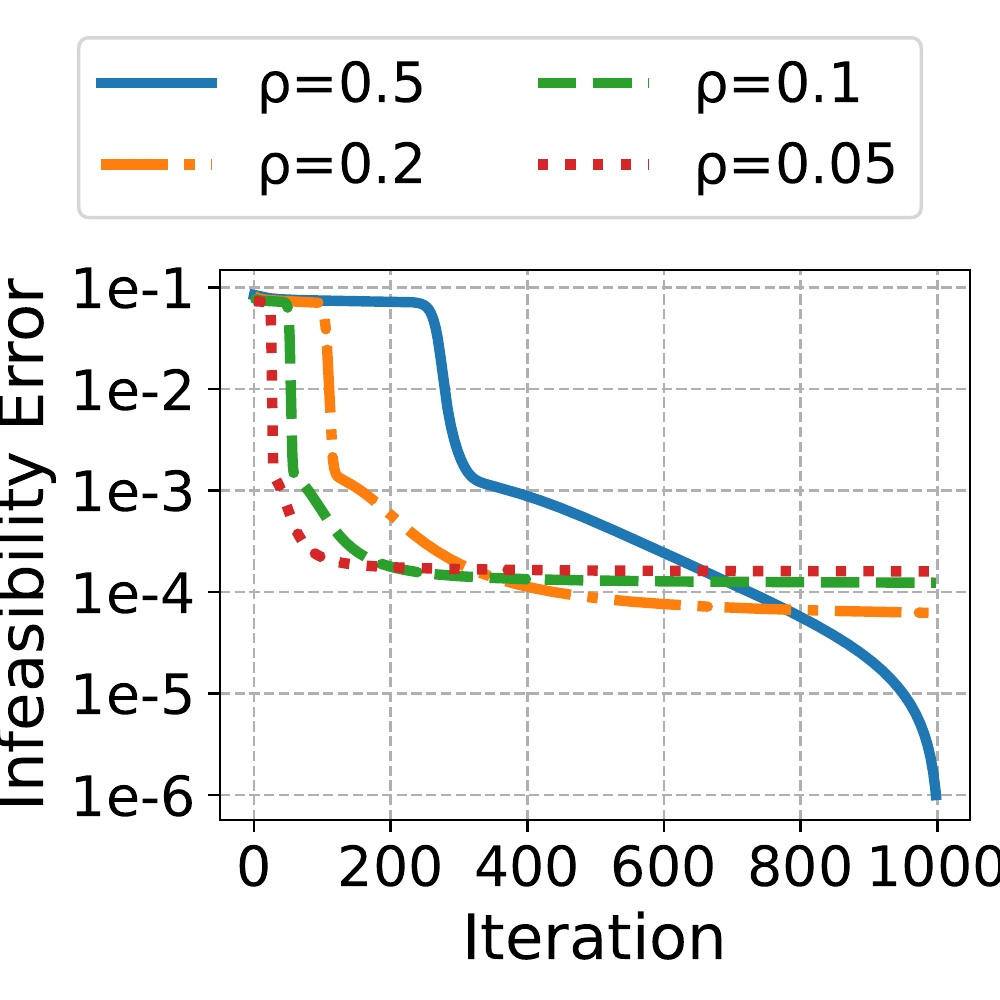}} 
\caption{(a) Sensitivity of the noise level and graph scale on the synthetic graph alignment database. (b) Sensitivity of BAPG's step size $\rho$. The infeasibility error  is defined as $\|\pi^k-w^k\|_2$. }\label{fig:noise}
\end{figure*}

\paragraph{GPU Acceleration of BAPG.} We conduct experiment results to further justify that BAPG is GPU-friendly. In Table \ref{tab:gpu}, We compare the acceleration ratio (i.e., CPU wall-clock time divided by GPU wall-clock time) of BAPG, eBPG, and BPG on two large-scale graph alignment datasets using the same computing server. For eBPG, we use the official implementation in the PythonOT package, which supports running on GPU. For BPG, we implement the GPU version by ourselves using Pytorch. We can find that BAPG has a much higher acceleration ratio on the GPU compared with BPG and eBPG.

\section{Closing Remark}
In this paper, we conduct a systematic investigation on developing provably efficient algorithms for the Gromov-Wasserstein distance computation. Two theoretically solid algorithms, called BAPG and hBPG, have been proposed to tackle different graph learning tasks. Our theoretical results provide novel insights to help users to select the well-suited algorithm for their tasks. Technically, by exploiting the error bound
condition (i.e., never investigated in the GW literature before), we are able to conduct the convergence analysis of BAPG (i.e., a special case of  alternating projected descent method). 
For a broader viewpoint, our work also takes the first step and provides a rather new path to study the alternating projected descent method for
general nonconvex problems. To the best of our knowledge, it is still an open question even in the optimization community.
Unfortunately, the proposed single-loop algorithm BAPG still suffers from the cubic per-iteration computational complexity, which will further limit its applications in real-world large-scale settings. 
A natural future direction is to consider sparse and low-rank structures of the matching matrix to decrease the per-iteration cost and further speed up our methods~\citep{scetbon2021linear}. Moreover, for simplicity, we assume that $D_X$ and $D_Y$ are two symmetric distance matrices. However, the proposed BAPG can  be potentially applied to the non-symmetric case as the Luo-tseng error bound condition still holds and algorithmic steps keep the same.

\section*{Acknowledgement}
 Material in this paper is based upon work supported by the Air Force Office
of Scientific Research under award number FA9550-20-1-0397. Additional support is gratefully
acknowledged from NSF grants 1915967 and 2118199. Jianheng Tang and Jia Li are supported by NSFC Grant 62206067 and Guangzhou-HKUST(GZ) Joint Funding Scheme. Anthony Man-Cho So is supported by in part by the Hong Kong RGC GRF project CUHK 14203920.

\newpage
\bibliography{main}

\begin{thebibliography}{56}
\providecommand{\natexlab}[1]{#1}
\providecommand{\url}[1]{\texttt{#1}}
\expandafter\ifx\csname urlstyle\endcsname\relax
  \providecommand{\doi}[1]{doi: #1}\else
  \providecommand{\doi}{doi: \begingroup \urlstyle{rm}\Url}\fi

\bibitem[Abrishami et~al.(2020)Abrishami, Guillen, Rule, Schutzman, Solomon,
  Weighill, and Wu]{abrishami2020geometry}
Tara Abrishami, Nestor Guillen, Parker Rule, Zachary Schutzman, Justin Solomon,
  Thomas Weighill, and Si~Wu.
\newblock Geometry of graph partitions via optimal transport.
\newblock \emph{SIAM Journal on Scientific Computing}, 42\penalty0
  (5):\penalty0 A3340--A3366, 2020.

\bibitem[Attouch et~al.(2010)Attouch, Bolte, Redont, and
  Soubeyran]{attouch2010proximal}
H{\'e}dy Attouch, J{\'e}r{\^o}me Bolte, Patrick Redont, and Antoine Soubeyran.
\newblock Proximal alternating minimization and projection methods for
  nonconvex problems: An approach based on the kurdyka-{\l}ojasiewicz
  inequality.
\newblock \emph{Mathematics of operations research}, 35\penalty0 (2):\penalty0
  438--457, 2010.

\bibitem[Attouch et~al.(2013)Attouch, Bolte, and
  Svaiter]{attouch2013convergence}
Hedy Attouch, J{\'e}r{\^o}me Bolte, and Benar~Fux Svaiter.
\newblock Convergence of descent methods for semi-algebraic and tame problems:
  proximal algorithms, forward--backward splitting, and regularized
  gauss--seidel methods.
\newblock \emph{Mathematical Programming}, 137\penalty0 (1):\penalty0 91--129,
  2013.

\bibitem[Bauschke et~al.(2017)Bauschke, Bolte, and
  Teboulle]{bauschke2017descent}
Heinz~H Bauschke, J{\'e}r{\^o}me Bolte, and Marc Teboulle.
\newblock A descent lemma beyond lipschitz gradient continuity: first-order
  methods revisited and applications.
\newblock \emph{Mathematics of Operations Research}, 42\penalty0 (2):\penalty0
  330--348, 2017.

\bibitem[Beck(2017)]{beck2017first}
Amir Beck.
\newblock \emph{First-order methods in optimization}, volume~25.
\newblock SIAM, 2017.

\bibitem[Benamou et~al.(2015)Benamou, Carlier, Cuturi, Nenna, and
  Peyr{\'e}]{benamou2015iterative}
Jean-David Benamou, Guillaume Carlier, Marco Cuturi, Luca Nenna, and Gabriel
  Peyr{\'e}.
\newblock Iterative bregman projections for regularized transportation
  problems.
\newblock \emph{SIAM Journal on Scientific Computing}, 37\penalty0
  (2):\penalty0 A1111--A1138, 2015.

\bibitem[Blondel et~al.(2008)Blondel, Guillaume, Lambiotte, and
  Lefebvre]{blondel2008fast}
Vincent~D Blondel, Jean-Loup Guillaume, Renaud Lambiotte, and Etienne Lefebvre.
\newblock Fast unfolding of communities in large networks.
\newblock \emph{Journal of statistical mechanics: theory and experiment},
  2008\penalty0 (10):\penalty0 P10008, 2008.

\bibitem[Bunne et~al.(2019)Bunne, Alvarez-Melis, Krause, and
  Jegelka]{bunne2019learning}
Charlotte Bunne, David Alvarez-Melis, Andreas Krause, and Stefanie Jegelka.
\newblock Learning generative models across incomparable spaces.
\newblock In \emph{International Conference on Machine Learning}, pages
  851--861. PMLR, 2019.

\bibitem[Chen et~al.(2020)Chen, Heimann, Vahedian, and
  Koutra]{chen2020consistent}
Xiyuan Chen, Mark Heimann, Fatemeh Vahedian, and Danai Koutra.
\newblock Consistent network alignment via proximity-preserving node embedding.
\newblock \emph{arXiv preprint arXiv:2005.04725}, 2020.

\bibitem[Cho et~al.(2010)Cho, Lee, and Lee]{rrwm}
Minsu Cho, Jungmin Lee, and Kyoung~Mu Lee.
\newblock Reweighted random walks for graph matching.
\newblock In \emph{European conference on Computer vision}, pages 492--505.
  Springer, 2010.

\bibitem[Chowdhury and M{\'e}moli(2019)]{chowdhury2019gromov}
Samir Chowdhury and Facundo M{\'e}moli.
\newblock The gromov--wasserstein distance between networks and stable network
  invariants.
\newblock \emph{Information and Inference: A Journal of the IMA}, 8\penalty0
  (4):\penalty0 757--787, 2019.

\bibitem[Chowdhury and Needham(2021)]{chowdhury2021generalized}
Samir Chowdhury and Tom Needham.
\newblock Generalized spectral clustering via gromov-wasserstein learning.
\newblock In \emph{International Conference on Artificial Intelligence and
  Statistics}, pages 712--720. PMLR, 2021.

\bibitem[Clauset et~al.(2004)Clauset, Newman, and Moore]{clauset2004finding}
Aaron Clauset, Mark~EJ Newman, and Cristopher Moore.
\newblock Finding community structure in very large networks.
\newblock \emph{Physical review E}, 70\penalty0 (6):\penalty0 066111, 2004.

\bibitem[Cuturi(2013)]{cuturi2013sinkhorn}
Marco Cuturi.
\newblock Sinkhorn distances: Lightspeed computation of optimal transport.
\newblock \emph{Advances in neural information processing systems},
  26:\penalty0 2292--2300, 2013.

\bibitem[Flamary et~al.(2021)Flamary, Courty, Gramfort, Alaya, Boisbunon,
  Chambon, Chapel, Corenflos, Fatras, Fournier, Gautheron, Gayraud, Janati,
  Rakotomamonjy, Redko, Rolet, Schutz, Seguy, Sutherland, Tavenard, Tong, and
  Vayer]{flamary2021pot}
R{\'e}mi Flamary, Nicolas Courty, Alexandre Gramfort, Mokhtar~Z. Alaya,
  Aur{\'e}lie Boisbunon, Stanislas Chambon, Laetitia Chapel, Adrien Corenflos,
  Kilian Fatras, Nemo Fournier, L{\'e}o Gautheron, Nathalie~T.H. Gayraud,
  Hicham Janati, Alain Rakotomamonjy, Ievgen Redko, Antoine Rolet, Antony
  Schutz, Vivien Seguy, Danica~J. Sutherland, Romain Tavenard, Alexander Tong,
  and Titouan Vayer.
\newblock Pot: Python optimal transport.
\newblock \emph{Journal of Machine Learning Research}, 22\penalty0
  (78):\penalty0 1--8, 2021.
\newblock URL \url{http://jmlr.org/papers/v22/20-451.html}.

\bibitem[Gao et~al.(2021)Gao, Huang, and Li]{gao2021unsupervised}
Ji~Gao, Xiao Huang, and Jundong Li.
\newblock Unsupervised graph alignment with wasserstein distance discriminator.
\newblock In \emph{Proceedings of the 27th ACM SIGKDD Conference on Knowledge
  Discovery \& Data Mining}, pages 426--435, 2021.

\bibitem[Hanzely et~al.(2021)Hanzely, Richtarik, and
  Xiao]{hanzely2021accelerated}
Filip Hanzely, Peter Richtarik, and Lin Xiao.
\newblock Accelerated bregman proximal gradient methods for relatively smooth
  convex optimization.
\newblock \emph{Computational Optimization and Applications}, 79\penalty0
  (2):\penalty0 405--440, 2021.

\bibitem[Hou et~al.(2013)Hou, Zhou, So, and Luo]{hou2013linear}
Ke~Hou, Zirui Zhou, Anthony Man-Cho So, and Zhi-Quan Luo.
\newblock On the linear convergence of the proximal gradient method for trace
  norm regularization.
\newblock In \emph{NIPS}, pages 710--718, 2013.

\bibitem[Jacobson et~al.(2018)]{jacobson2018gptoolbox}
Alec Jacobson et~al.
\newblock gptoolbox: Geometry processing toolbox.
\newblock \emph{ONLINE: http://github. com/alecjacobson/gptoolbox}, 2018.

\bibitem[Jaggi(2013)]{jaggi2013revisiting}
Martin Jaggi.
\newblock Revisiting frank-wolfe: Projection-free sparse convex optimization.
\newblock In \emph{International Conference on Machine Learning}, pages
  427--435. PMLR, 2013.

\bibitem[Krichene et~al.(2015)Krichene, Krichene, and
  Bayen]{krichene2015efficient}
Walid Krichene, Syrine Krichene, and Alexandre Bayen.
\newblock Efficient bregman projections onto the simplex.
\newblock In \emph{2015 54th IEEE Conference on Decision and Control (CDC)},
  pages 3291--3298. IEEE, 2015.

\bibitem[Lacoste-Julien(2016)]{lacoste2016convergence}
Simon Lacoste-Julien.
\newblock Convergence rate of frank-wolfe for non-convex objectives.
\newblock \emph{arXiv preprint arXiv:1607.00345}, 2016.

\bibitem[Lacoste-Julien et~al.(2006)Lacoste-Julien, Taskar, Klein, and
  Jordan]{lacoste2006word}
Simon Lacoste-Julien, Ben Taskar, Dan Klein, and Michael Jordan.
\newblock Word alignment via quadratic assignment.
\newblock 2006.

\bibitem[Lawler(1963)]{lawler1963quadratic}
Eugene~L Lawler.
\newblock The quadratic assignment problem.
\newblock \emph{Management science}, 9\penalty0 (4):\penalty0 586--599, 1963.

\bibitem[Leordeanu and Hebert(2005)]{sm}
Marius Leordeanu and Martial Hebert.
\newblock A spectral technique for correspondence problems using pairwise
  constraints.
\newblock In \emph{International Conference on Computer Vision}, pages
  1482--1489. IEEE, 2005.

\bibitem[Leordeanu et~al.(2009)Leordeanu, Hebert, and Sukthankar]{ipfp}
Marius Leordeanu, Martial Hebert, and Rahul Sukthankar.
\newblock An integer projected fixed point method for graph matching and map
  inference.
\newblock \emph{Advances in neural information processing systems}, 22, 2009.

\bibitem[Li and Pong(2018)]{li2018calculus}
Guoyin Li and Ting~Kei Pong.
\newblock Calculus of the exponent of {Kurdyka-{\L}ojasiewicz} inequality and
  its applications to linear convergence of first-order methods.
\newblock \emph{Foundations of Computational Mathematics}, 18\penalty0
  (5):\penalty0 1199--1232, 2018.

\bibitem[Li et~al.(2020)Li, Sun, and Toh]{li2020efficient}
Xudong Li, Defeng Sun, and Kim-Chuan Toh.
\newblock On the efficient computation of a generalized jacobian of the
  projector over the birkhoff polytope.
\newblock \emph{Mathematical Programming}, 179\penalty0 (1-2):\penalty0
  419--446, 2020.

\bibitem[Luo and Tseng(1992)]{luo1992error}
Zhi-Quan Luo and Paul Tseng.
\newblock Error bound and convergence analysis of matrix splitting algorithms
  for the affine variational inequality problem.
\newblock \emph{SIAM Journal on Optimization}, 2\penalty0 (1):\penalty0 43--54,
  1992.

\bibitem[Mai et~al.(2021)Mai, Lindb{\"a}ck, and Johansson]{mai2021fast}
Vien~V Mai, Jacob Lindb{\"a}ck, and Mikael Johansson.
\newblock A fast and accurate splitting method for optimal transport: Analysis
  and implementation.
\newblock \emph{arXiv preprint arXiv:2110.11738}, 2021.

\bibitem[M{\'e}moli(2009)]{memoli2009spectral}
Facundo M{\'e}moli.
\newblock Spectral gromov-wasserstein distances for shape matching.
\newblock In \emph{2009 IEEE 12th International Conference on Computer Vision
  Workshops, ICCV Workshops}, pages 256--263. IEEE, 2009.

\bibitem[M{\'e}moli(2011)]{memoli2011gromov}
Facundo M{\'e}moli.
\newblock Gromov--wasserstein distances and the metric approach to object
  matching.
\newblock \emph{Foundations of computational mathematics}, 11\penalty0
  (4):\penalty0 417--487, 2011.

\bibitem[M{\'e}moli(2014)]{memoli2014gromov}
Facundo M{\'e}moli.
\newblock The gromov--wasserstein distance: A brief overview.
\newblock \emph{Axioms}, 3\penalty0 (3):\penalty0 335--341, 2014.

\bibitem[M{\'e}moli and Sapiro(2004)]{memoli2004comparing}
Facundo M{\'e}moli and Guillermo Sapiro.
\newblock Comparing point clouds.
\newblock In \emph{Proceedings of the 2004 Eurographics/ACM SIGGRAPH symposium
  on Geometry processing}, pages 32--40, 2004.

\bibitem[Nedi{\'c}(2011)]{nedic2011random}
Angelia Nedi{\'c}.
\newblock Random algorithms for convex minimization problems.
\newblock \emph{Mathematical programming}, 129\penalty0 (2):\penalty0 225--253,
  2011.

\bibitem[Nocedal and Wright(2006)]{nocedal2006numerical}
Jorge Nocedal and Stephen Wright.
\newblock \emph{Numerical optimization}.
\newblock Springer Science \& Business Media, 2006.

\bibitem[Peyr{\'e} et~al.(2016)Peyr{\'e}, Cuturi, and Solomon]{peyre2016gromov}
Gabriel Peyr{\'e}, Marco Cuturi, and Justin Solomon.
\newblock Gromov-wasserstein averaging of kernel and distance matrices.
\newblock In \emph{International Conference on Machine Learning}, pages
  2664--2672. PMLR, 2016.

\bibitem[Peyr{\'e} et~al.(2019)Peyr{\'e}, Cuturi,
  et~al.]{peyre2019computational}
Gabriel Peyr{\'e}, Marco Cuturi, et~al.
\newblock Computational optimal transport: With applications to data science.
\newblock \emph{Foundations and Trends{\textregistered} in Machine Learning},
  11\penalty0 (5-6):\penalty0 355--607, 2019.

\bibitem[Rosvall and Bergstrom(2008)]{rosvall2008maps}
Martin Rosvall and Carl~T Bergstrom.
\newblock Maps of random walks on complex networks reveal community structure.
\newblock \emph{Proceedings of the national academy of sciences}, 105\penalty0
  (4):\penalty0 1118--1123, 2008.

\bibitem[Scetbon et~al.(2021)Scetbon, Peyr{\'e}, and Cuturi]{scetbon2021linear}
Meyer Scetbon, Gabriel Peyr{\'e}, and Marco Cuturi.
\newblock Linear-time gromov wasserstein distances using low rank couplings and
  costs.
\newblock \emph{arXiv preprint arXiv:2106.01128}, 2021.

\bibitem[S{\'e}journ{\'e} et~al.(2021)S{\'e}journ{\'e}, Vialard, and
  Peyr{\'e}]{sejourne2021unbalanced}
Thibault S{\'e}journ{\'e}, Fran{\c{c}}ois-Xavier Vialard, and Gabriel
  Peyr{\'e}.
\newblock The unbalanced gromov wasserstein distance: Conic formulation and
  relaxation.
\newblock \emph{Advances in Neural Information Processing Systems},
  34:\penalty0 8766--8779, 2021.

\bibitem[Solomon et~al.(2016)Solomon, Peyr{\'e}, Kim, and
  Sra]{solomon2016entropic}
Justin Solomon, Gabriel Peyr{\'e}, Vladimir~G Kim, and Suvrit Sra.
\newblock Entropic metric alignment for correspondence problems.
\newblock \emph{ACM Transactions on Graphics (TOG)}, 35\penalty0 (4):\penalty0
  1--13, 2016.

\bibitem[Titouan et~al.(2019)Titouan, Courty, Tavenard, and
  Flamary]{titouan2019optimal}
Vayer Titouan, Nicolas Courty, Romain Tavenard, and R{\'e}mi Flamary.
\newblock Optimal transport for structured data with application on graphs.
\newblock In \emph{International Conference on Machine Learning}, pages
  6275--6284. PMLR, 2019.

\bibitem[Vayer et~al.(2018)Vayer, Chapel, Flamary, Tavenard, and
  Courty]{vayer2018fused}
Titouan Vayer, Laetita Chapel, R{\'e}mi Flamary, Romain Tavenard, and Nicolas
  Courty.
\newblock Fused gromov-wasserstein distance for structured objects: theoretical
  foundations and mathematical properties.
\newblock \emph{arXiv preprint arXiv:1811.02834}, 2018.

\bibitem[Vayer et~al.(2019)Vayer, Flamary, Tavenard, Chapel, and
  Courty]{vayer2019sliced}
Titouan Vayer, R{\'e}mi Flamary, Romain Tavenard, Laetitia Chapel, and Nicolas
  Courty.
\newblock Sliced gromov-wasserstein.
\newblock In \emph{NeurIPS 2019-Thirty-third Conference on Neural Information
  Processing Systems}, volume~32, 2019.

\bibitem[Vincent-Cuaz et~al.(2021{\natexlab{a}})Vincent-Cuaz, Flamary, Corneli,
  Vayer, and Courty]{vincent2021semi}
C{\'e}dric Vincent-Cuaz, R{\'e}mi Flamary, Marco Corneli, Titouan Vayer, and
  Nicolas Courty.
\newblock Semi-relaxed gromov wasserstein divergence with applications on
  graphs.
\newblock \emph{arXiv preprint arXiv:2110.02753}, 2021{\natexlab{a}}.

\bibitem[Vincent-Cuaz et~al.(2021{\natexlab{b}})Vincent-Cuaz, Vayer, Flamary,
  Corneli, and Courty]{vincent2021online}
C{\'e}dric Vincent-Cuaz, Titouan Vayer, R{\'e}mi Flamary, Marco Corneli, and
  Nicolas Courty.
\newblock Online graph dictionary learning.
\newblock \emph{arXiv preprint arXiv:2102.06555}, 2021{\natexlab{b}}.

\bibitem[Vinh et~al.(2010)Vinh, Epps, and Bailey]{vinh2010information}
Nguyen~Xuan Vinh, Julien Epps, and James Bailey.
\newblock Information theoretic measures for clusterings comparison: Variants,
  properties, normalization and correction for chance.
\newblock \emph{The Journal of Machine Learning Research}, 11:\penalty0
  2837--2854, 2010.

\bibitem[Wang and Bertsekas(2016)]{wang2016stochastic}
Mengdi Wang and Dimitri~P Bertsekas.
\newblock Stochastic first-order methods with random constraint projection.
\newblock \emph{SIAM Journal on Optimization}, 26\penalty0 (1):\penalty0
  681--717, 2016.

\bibitem[Xu et~al.(2019{\natexlab{a}})Xu, Luo, and Carin]{xu2019scalable}
Hongteng Xu, Dixin Luo, and Lawrence Carin.
\newblock Scalable gromov-wasserstein learning for graph partitioning and
  matching.
\newblock \emph{Advances in neural information processing systems},
  32:\penalty0 3052--3062, 2019{\natexlab{a}}.

\bibitem[Xu et~al.(2019{\natexlab{b}})Xu, Luo, Zha, and Duke]{xu2019gromov}
Hongteng Xu, Dixin Luo, Hongyuan Zha, and Lawrence~Carin Duke.
\newblock Gromov-wasserstein learning for graph matching and node embedding.
\newblock In \emph{International conference on machine learning}, pages
  6932--6941. PMLR, 2019{\natexlab{b}}.

\bibitem[Xu et~al.(2021)Xu, Luo, Carin, and Zha]{xu2021learning}
Hongteng Xu, Dixin Luo, Lawrence Carin, and Hongyuan Zha.
\newblock Learning graphons via structured gromov-wasserstein barycenters.
\newblock In \emph{Proceedings of the AAAI Conference on Artificial
  Intelligence}, volume~35, pages 10505--10513, 2021.

\bibitem[Xu et~al.(2022)Xu, Liu, Luo, and Carin]{xu2022representing}
Hongteng Xu, Jiachang Liu, Dixin Luo, and Lawrence Carin.
\newblock Representing graphs via gromov-wasserstein factorization.
\newblock \emph{IEEE Transactions on Pattern Analysis and Machine
  Intelligence}, 2022.

\bibitem[Zhang and Luo(2022)]{zhang2020global}
Jiawei Zhang and Zhi-Quan Luo.
\newblock A global dual error bound and its application to the analysis of
  linearly constrained nonconvex optimization.
\newblock \emph{SIAM Journal on Optimization}, 32\penalty0 (3):\penalty0
  2319--2346, 2022.

\bibitem[Zhang et~al.(2021)Zhang, Tong, Jin, Xia, and Guo]{zhang2021balancing}
Si~Zhang, Hanghang Tong, Long Jin, Yinglong Xia, and Yunsong Guo.
\newblock Balancing consistency and disparity in network alignment.
\newblock In \emph{Proceedings of the 27th ACM SIGKDD Conference on Knowledge
  Discovery \& Data Mining}, pages 2212--2222, 2021.

\bibitem[Zhou and So(2017)]{zhou2017unified}
Zirui Zhou and Anthony Man-Cho So.
\newblock A unified approach to error bounds for structured convex optimization
  problems.
\newblock \emph{Mathematical Programming}, 165\penalty0 (2):\penalty0 689--728,
  2017.

\end{thebibliography}

\end{document}